\documentclass[journal]{IEEEtran}

\usepackage{amsmath,amssymb,amsthm,mathtools}
\usepackage{bm}
\usepackage{graphicx}
\usepackage{booktabs}
\usepackage{enumitem}
\usepackage{algorithm}
\usepackage{algpseudocode}
\usepackage{tikz}
\usetikzlibrary{arrows.meta, positioning, calc, shapes.geometric, fit, backgrounds}
\usepackage{hyperref}
\usepackage{cleveref}

\newtheorem{axiom}{Axiom}

\newtheorem{theorem}{Theorem}

\newtheorem{proposition}[theorem]{Proposition}
\newtheorem{corollary}[theorem]{Corollary}
\theoremstyle{definition}
\newtheorem{definition}{Definition}
\theoremstyle{remark}
\newtheorem{remark}{Remark}

\newcommand{\R}{\mathbb{R}}
\newcommand{\N}{\mathbb{N}}
\newcommand{\E}{\mathbb{E}}
\newcommand{\PP}{\mathbb{P}}
\newcommand{\simplex}{\Delta}
\newcommand{\Ssp}{\mathcal{S}}
\newcommand{\Asp}{\mathcal{A}}
\newcommand{\Osp}{\mathcal{O}}
\newcommand{\Hsp}{\mathcal{H}}
\newcommand{\Xsp}{\mathcal{X}}
\newcommand{\BSE}{\textsc{BSE}}
\newcommand{\LLM}{\textsc{LLM}}
\DeclareMathOperator*{\argmax}{arg\,max}

\begin{document}

\title{Belief-State Engine: Augmenting LLMs for Principled Planning Under Partial Observability}

\author{%
  Arnab~Chattopadhayay
  \thanks{A.~Chattopadhayay is an Independent Researcher
          (UCL~Alumni), Bangalore, India
          (e-mail: \href{mailto:arnab.chattopadhayay@uclmail.net}{arnab.chattopadhayay@uclmail.net}).}%
  ~and~Debdipta~Halder%
  \thanks{D.~Halder is an Independent Researcher
          (IIT-Kharagpur Alumni), Bangalore, India
          (e-mail: \href{mailto:haldev50@gmail.com}{haldev50@gmail.com}).}%
  \thanks{Manuscript prepared April 2026. Code, environment specifications,
          prompt templates, and paired-seed logs accompanying this
          preprint are available at \url{https://github.com/debdipta-h/bse-llm}.}
}

\markboth{Preprint, April~2026}%
         {Chattopadhayay \& Halder: Belief-State Engine for LLM Planning Under Partial Observability}

\maketitle

\begin{abstract}
Large language model agents produce fluent action sequences across
a wide range of tasks, yet they fail in characteristic ways once
the environment becomes partially observable. Ambiguous feedback
pushes them into premature commitments. A single informative
observation can collapse their uncertainty onto the wrong
hypothesis. Policies drift as the history grows. We trace these
symptoms to a common structural cause. An LLM agent, as commonly
deployed, is a history-conditioned policy with no explicit belief
over hidden state.

We propose an architectural fix. The Belief-State Engine (BSE)
is an inference module placed outside the LLM. It maintains a
Bayesian posterior over the latent states of a given POMDP model,
and at each decision step it exposes only that posterior to the
LLM. The raw action-observation log is not shown. We set out a
minimal four-axiom specification of what a belief-consistent
internal state must satisfy, and prove that the LLM paired with
the BSE is a sound Markov policy on the belief MDP induced by
the underlying POMDP. It therefore inherits the Bellman
optimality guarantees of classical POMDP theory, provided the
LLM is never exposed to the raw history.

We evaluate the architecture on the Tiger POMDP and a red-team
attack-graph task, against six baselines: a reactive LLM,
Chain-of-Thought, ReAct, a natural-language belief tracker,
QMDP, and POMCP. Across both domains, the BSE-augmented agent
improves task return, belief calibration, and decision
consistency. Ten targeted ablations isolate the contribution of
each architectural choice, and a replication on an open-weights
backbone confirms that the effect is not specific to any one
model. Code, environment specifications, prompt templates, and
seed logs accompany this preprint.
\end{abstract}


\section{Introduction}
\label{sec:intro}

\IEEEPARstart{L}{arge} language models now sit at the centre of
a growing class of autonomous agents. Systems such as
ReAct~\cite{yao2022react}, Reflexion~\cite{shinn2023reflexion},
Tree-of-Thoughts~\cite{yao2023tree}, Voyager~\cite{wang2023voyager},
and SWE-agent~\cite{yang2024swe} wrap an LLM inside a loop that
turns text-level reasoning into multi-step action sequences, over
settings that range from embodied simulation to software repair
and web navigation. These systems share a common failure
profile~\cite{valmeekam2023planning, liu2024agentbench,
xi2023rise}. They work well when the next action can be decided
from the latest observation. They stumble once the environment
becomes partially observable. Feedback that is delayed, or
ambiguous, or actively deceptive, tends to expose a class of
failures that Chain-of-Thought prompting does not resolve.

In our view the issue is architectural, not a question of
reasoning depth. Keeping a calibrated belief over hidden states,
and then picking actions with respect to that belief, is what
Bayesian filtering gives you. It is not what Chain-of-Thought
prompting, a longer context window, or a reflection loop confers
on an LLM. These mechanisms accumulate text. They do not
accumulate probability mass. Two histories that would yield the
same Bayesian posterior can elicit very different action
distributions from an LLM whenever their surface texts differ,
and that already breaks the most basic consistency requirement of
a belief-measurable policy. The behaviour this produces has been
reported many times: premature commitment when evidence is still
ambiguous, over-confident collapse onto one hypothesis after a
single informative observation, and policy drift as the context
grows.

Our fix is to stop asking the LLM to act as a planner under
partial observability. We wrap it in a two-module system. An
external inference component, which we call the Belief-State
Engine, holds the epistemic state. The LLM is restricted to
selecting actions conditioned on that state. Concretely, the BSE
runs a Bayesian filter over a POMDP model $(T, Z)$ of the
environment, and at each step it hands the LLM a normalised
belief $b_t \in \simplex(\Ssp)$ as the only decision-time
context. The raw trace of past actions and observations is not
shown to the LLM.

This design has a precise mathematical backing. Under a minimal
four-axiom specification of belief-consistent internal state, the
LLM paired with the BSE is a Markov policy on the belief MDP
induced by the underlying POMDP. It therefore inherits the
Bellman optimality theorems of classical POMDP
theory~\cite{astrom1965, smallwood1973, puterman1994}. The axioms
pin down what qualifies as a belief-consistent internal state.
The theorems lift the architectural choice into a compositionality
guarantee with a well-defined failure mode: if the LLM is shown
the raw history at decision time, Axiom~A4 is violated and the
guarantee falls with it.

\paragraph*{Contributions}
This paper makes four contributions.
\begin{enumerate}[leftmargin=*,itemsep=2pt,topsep=2pt]
\item \emph{Architecture.} We introduce the Belief-State Engine,
an LLM-external and model-agnostic inference module that
maintains a POMDP-grounded belief and interfaces with any
belief-measurable policy, including an LLM-parameterised one
(\cref{sec:architecture}).

\item \emph{Theory.} We set out a minimal four-axiom
characterisation of belief-consistent internal state and prove six
theorems: existence and minimality of the canonical posterior,
uniqueness of the Bayes update, value equivalence with the belief
MDP, ambiguity preservation under bisimulation, and soundness of
the LLM-BSE composition (\cref{sec:axioms}; full proofs in
Appendix~A).

\item \emph{Empirical study.} We evaluate the BSE on two
environments, the Tiger POMDP and a red-team attack-graph task,
against six baselines (reactive, Chain-of-Thought, ReAct,
natural-language belief tracker, QMDP, POMCP). Four metric
families are reported: task return, belief calibration, decision
consistency, and compute cost. Ten architectural ablations and an
open-weights robustness replication accompany the main results
(\cref{sec:methodology,sec:results}).

\item \emph{Artefact release.} Code, prompt templates, environment
specifications, reward matrices, and paired-seed logs are
published with this preprint.
\end{enumerate}

\paragraph*{Paper organisation}
\Cref{sec:related} reviews POMDP theory, LLM-agent failure modes,
and related hybrid work. \Cref{sec:preliminaries} fixes notation.
\Cref{sec:axioms} sets out the axiomatic foundations.
\Cref{sec:architecture} describes the BSE. \Cref{sec:methodology}
details the experimental methodology. \Cref{sec:results} reports
results. \Cref{sec:limitations} discusses limitations and
approximate belief representations. \Cref{sec:conclusion}
concludes.

\section{Background and Related Work}
\label{sec:related}

\subsection{POMDPs and Belief-State Theory}
\label{sec:related-pomdp}

Partially observable Markov decision processes, or POMDPs,
formalise sequential decision-making when the state is
latent~\cite{astrom1965, smallwood1973, kaelbling1998planning}. A
POMDP $M = (\Ssp,\Asp,\Osp,T,Z,R,\gamma,\mu_0)$ extends the
standard MDP by a finite observation space $\Osp$ and an
observation kernel $Z(o\mid s,a)$. Its optimal policy can be
written as a function of the belief state $b_t \in
\simplex(\Ssp)$, the posterior over latent states given the
interaction history~\cite{sondik1971, smallwood1973}. The belief
is a sufficient statistic for optimal decision-making~\cite{striebel1965,
puterman1994}. Updating it recursively via the two-step Bayes
filter, prediction followed by observation-conditioned correction,
turns the POMDP into a fully observable Markov decision process on
$\simplex(\Ssp)$, the belief MDP.

Exact solution of belief MDPs is PSPACE-hard in the finite-horizon
case, and undecidable in the infinite-horizon
case~\cite{papadimitriou1987}. Several decades of research have
produced approximate solvers that scale to useful state sizes.
Point-based value iteration~\cite{pineau2003} and
SARSOP~\cite{kurniawati2008} restrict value computation to a
representative subset of reachable beliefs. QMDP~\cite{littman1995}
computes a fast heuristic by treating the environment as fully
observable after the current step. POMCP~\cite{silver2010} runs
Monte Carlo tree search in belief space via sampled rollouts. All
four methods assume access to the model $(T, Z)$, or at least to
a simulator of it. Our empirical setup makes the same assumption.
Learning the belief-state model from data is a separate research
thread that we discuss in \cref{sec:limitations}.

\subsection{LLM Agents and Their Failure Modes}
\label{sec:related-llm-agents}

Starting with ReAct~\cite{yao2022react} and continuing through
Reflexion~\cite{shinn2023reflexion},
Tree-of-Thoughts~\cite{yao2023tree}, Voyager~\cite{wang2023voyager},
SWE-agent~\cite{yang2024swe}, and the Cognitive Architectures for
Language Agents survey~\cite{sumers2024cognitive}, the working
template for an LLM agent has been to iterate an LLM over a
growing text log of past actions and observations, interleaved
with reasoning traces, tool calls, and retrieval. The template
works well when the next action can be decided from the current
observation and the cost of long-context reasoning is affordable.

Systematic evaluations have mapped the limits of that template.
Valmeekam and co-authors~\cite{valmeekam2023planning,
valmeekam2024planbench} show that LLMs do poorly on classical
planning problems even with Chain-of-Thought. Liu et
al.~\cite{liu2024agentbench} report steep performance drops on
agent benchmarks once observations are noisy or delayed. The
standard interpretation in the literature is that the LLM lacks a
world model, or a calibrated uncertainty estimate. Our own
position is more specific. The missing component is a sufficient
statistic of history. No amount of reasoning over raw text recovers
what Bayesian conditioning gives automatically.

\subsection{LLM-POMDP Hybrids and Belief Tracking}
\label{sec:related-hybrids}

A smaller but growing line of work couples LLMs to some form of
explicit state representation. Reasoning via
Planning~\cite{hao2023rap} uses an LLM as a world model inside
MCTS-style search, but the resulting belief is implicit in the
tree. Du et al.~\cite{du2023guiding} have LLMs track symbolic task
state for goal-conditioned learning. Xie et
al.~\cite{xie2022explanation} and related work interpret
in-context learning itself as implicit Bayesian inference, though
the posterior is never surfaced. In the agent setting, a related
strand asks an LLM to act as its own belief-state approximator,
typically by narrating its uncertainty in natural language rather
than maintaining an explicit probability distribution.

Three points distinguish the BSE from this line of work. First,
the belief is external and explicit. It lives as a probability
distribution over a finite latent state space, rather than inside
the LLM's activations or in a free-text paragraph. Second, the
architecture is compositional. Any LLM can be dropped into the
policy slot, and any belief-measurable policy is sound by
Theorem~\ref{thm:llm-soundness}. Third, the belief is auditable.
At every decision point the full posterior $b_t$ is available to
the operator, which supports debugging, monitoring, and guarantee
checks in a way that implicit or free-text representations cannot
match. We include a natural-language belief tracker as a direct
baseline in our experiments (\cref{sec:methodology}) precisely to
test whether the benefit we report comes from \emph{any} epistemic
state, or specifically from a \emph{probabilistic} one.

\subsection{Uncertainty and Calibration in LLMs}
\label{sec:related-calibration}

Outside the agent setting, a parallel literature asks whether LLMs
are calibrated. Kadavath et al.~\cite{kadavath2022language} find
that large models are approximately calibrated on factual
question answering. Kuhn et al.~\cite{kuhn2023semantic} extend
this to semantic uncertainty in free-form generation.
Distribution-free coverage guarantees on single-shot LLM outputs
are available through conformal
prediction~\cite{angelopoulos2023conformal}. These results all
concern an LLM's self-reported confidence on individual questions.
They do not address the distinct problem of maintaining a
calibrated posterior over a hidden environment state across a
multi-step trajectory. That is the problem the BSE targets.

\subsection{The Gap}
\label{sec:related-gap}

None of the lines reviewed above gives what the BSE gives: an
LLM-external, auditable, POMDP-grounded belief, maintained in
closed form, compositional with any LLM policy backbone, and
backed by a compositionality theorem that specifies the conditions
under which the composition inherits classical POMDP guarantees.
The rest of the paper develops that construction.
\Cref{sec:preliminaries} fixes notation. \Cref{sec:axioms}
supplies the axiomatic and theoretical foundation.
\Cref{sec:architecture} specifies the engine itself.


\section{Problem Setting and Preliminaries}
\label{sec:preliminaries}

We fix notation in this section and then state the central
observation that motivates the rest of the paper: an LLM deployed
as a history-conditioned policy is not a belief-measurable policy,
and that gap is what the Belief-State Engine closes.

\subsection{POMDPs and the Belief MDP}
\label{sec:prelim-pomdp}

Let $M = (\Ssp, \Asp, \Osp, T, Z, R, \gamma, \mu_0)$ be a finite
partially observable Markov decision
process~\cite{astrom1965,smallwood1973,kaelbling1998planning}. Here
$\Ssp$ is a finite set of latent states, $\Asp$ a finite action
set, and $\Osp$ a finite observation set. The transition kernel
$T : \Ssp \times \Asp \to \simplex(\Ssp)$ maps a state-action pair
to a distribution over next states, with
$T(s' \mid s,a) := \PP(S_{t+1}{=}s' \mid S_t{=}s, A_t{=}a)$. The
observation kernel $Z : \Ssp \times \Asp \to \simplex(\Osp)$ emits
$Z(o \mid s',a) := \PP(O_{t+1}{=}o \mid S_{t+1}{=}s', A_t{=}a)$
after the transition. The reward function
$R : \Ssp \times \Asp \to \R$ is bounded, the discount factor
$\gamma \in [0,1)$, and $\mu_0 \in \simplex(\Ssp)$ is the initial
state distribution. The agent does not observe $S_t$ at any step.

At time $t$ the agent has seen an interaction history
\begin{equation}
  h_t
  = (a_0, o_1, a_1, o_2, \dots, a_{t-1}, o_t)
  \in \Hsp_t,
  \label{eq:history}
\end{equation}
with $\Hsp_t := (\Asp \times \Osp)^t$ and $h_0 := \varnothing$.
Let $\Hsp := \bigcup_{t \geq 0} \Hsp_t$ denote the set of all
finite histories. A history-conditioned policy is a map
$\pi : \Hsp \to \simplex(\Asp)$. The induced trajectory
distribution factorises in the standard way:
\begin{equation}
\begin{split}
  \PP_\pi(S_{0:T}, A_{0:T-1}, O_{1:T})
  ={}& \mu_0(S_0) \prod_{t=0}^{T-1}
       \pi(A_t \mid h_t) \\
     & {}\times T(S_{t+1} \mid S_t, A_t)\,
         Z(O_{t+1} \mid S_{t+1}, A_t).
\end{split}
  \label{eq:trajectory}
\end{equation}
The objective is the expected discounted return
$J(\pi) := \E_\pi\!\big[\sum_{t=0}^{\infty} \gamma^{t}
  R(S_t, A_t)\big]$.

The belief at time $t$ is the posterior over the latent state
given the history:
\begin{equation}
  b_t(s) := \PP(S_t{=}s \mid h_t, \mu_0),
  \qquad b_t \in \simplex(\Ssp).
  \label{eq:belief-def}
\end{equation}
The belief admits the recursive Bayes-filter
update~\cite{sondik1971,striebel1965}. Given $b_t$, action $a_t$,
and observation $o_{t+1}$, the posterior at $t{+}1$ is
\begin{align}
  \bar b_{t+1}(s')
  &= \sum_{s \in \Ssp} T(s' \mid s, a_t)\, b_t(s),
  \label{eq:predict}\\
  b_{t+1}(s')
  &= \frac{Z(o_{t+1} \mid s', a_t)\, \bar b_{t+1}(s')}
          {\sum_{s'' \in \Ssp} Z(o_{t+1} \mid s'', a_t)\,
            \bar b_{t+1}(s'')}.
  \label{eq:update}
\end{align}
Equation~\eqref{eq:predict} is the prediction step.
Equation~\eqref{eq:update} is the observation-conditioned
correction. The two steps must be performed in order. Dropping the
prediction step, as the algorithm in the TechRxiv version of this
work inadvertently did, gives an incorrect posterior whenever
$T \neq I$. We return to this point when we specify the engine in
\cref{sec:architecture}.

Once the belief is introduced, the POMDP $M$ is equivalent in value
to a fully observable Markov decision process on
$\simplex(\Ssp)$, known as the belief MDP. Its transition kernel
is
\begin{equation}
  \tau(b' \mid b, a)
  = \sum_{o \in \Osp}
      \PP(o \mid b, a)\,
      \mathbf{1}\!\left[b' = \mathrm{Bayes}(b, a, o)\right],
  \label{eq:belief-mdp}
\end{equation}
where $\PP(o \mid b, a) = \sum_{s', s} Z(o \mid s', a)\,
T(s' \mid s, a)\, b(s)$ is the marginal observation probability and
$\mathrm{Bayes}(b, a, o)$ is the posterior given by
\eqref{eq:predict}--\eqref{eq:update}. The belief-MDP reward is
$r(b, a) := \sum_s b(s)\, R(s, a)$. The optimal value function
$V^*$ on $\simplex(\Ssp)$ satisfies the Bellman
equation~\cite{smallwood1973,puterman1994}
\begin{equation}
  V^*(b)
  = \max_{a \in \Asp}
    \Big[
      r(b, a) + \gamma \sum_{o \in \Osp}
      \PP(o \mid b, a)\,
      V^*\!\left(\mathrm{Bayes}(b, a, o)\right)
    \Big].
  \label{eq:bellman}
\end{equation}
The belief is therefore a sufficient statistic of history for the
purposes of optimal decision-making. Any policy that is a function
of $h_t$ alone, and that attains $J(\pi) = V^*(\mu_0)$, must be
expressible as a function of $b_t$ on the support of trajectories
induced by $\pi$. We use this fact repeatedly in \cref{sec:axioms}.

\subsection{LLMs as History-Conditioned Text Policies}
\label{sec:prelim-llm}

Let $\Xsp$ denote a token alphabet and let $\phi : \Hsp \to
\Xsp^{\!\star}$ be a deterministic serialiser that maps a history
$h_t$ to a finite token string, for example a system prompt
followed by a transcript of past actions and observations. Given a
decoding temperature and a sampling rule, a language model induces
a conditional distribution $q_{\LLM}(y \mid x)$ over token
continuations $y \in \Xsp^{\!\star}$ given any prompt $x \in
\Xsp^{\!\star}$. An action parser $\psi : \Xsp^{\!\star} \to \Asp
\cup \{\bot\}$ extracts an action from the LLM output, mapping
malformed continuations to a designated abstain symbol $\bot$,
which the deployment layer typically resolves by default action or
retry.

Composing these pieces gives a history-conditioned text policy
\begin{equation}
  \pi_{\LLM}(a \mid h_t)
  := \PP_{y \sim q_{\LLM}(\cdot \mid \phi(h_t))}
     \big[\psi(y) = a\big],
  \qquad a \in \Asp.
  \label{eq:pi-llm}
\end{equation}
This is the object that is instantiated, explicitly or implicitly,
by every LLM agent architecture
in~\cite{yao2022react,shinn2023reflexion,yao2023tree,
wang2023voyager,yang2024swe}. Chain-of-Thought prompting,
reflection, and scratchpad memories are all variations on the
serialiser $\phi$ and on the parser $\psi$. They do not, on their
own, change the type of the policy. It remains a map from
histories to action distributions.

The deployment pipeline of equation~\eqref{eq:pi-llm} is subject
to practical constraints that the belief formulation is not.
Context-window truncation replaces $\phi(h_t)$ by a lossy
prefix-or-summary for long $t$. Chat-style APIs introduce
provider-side non-determinism in $q_{\LLM}$ even at temperature
zero. The parser $\psi$ can fail on malformed outputs. None of
these are first-order to our argument, but they each reduce the
agent's effective access to the history below what the formalism
of equation~\eqref{eq:trajectory} assumes. We flag them here so
that the reader can distinguish implementation artefacts from the
structural point we develop next.

\subsection{Why History Conditioning Is Not Enough}
\label{sec:prelim-gap}

A natural reading of equation~\eqref{eq:pi-llm} is that
$\pi_{\LLM}$, by taking the full history as input, has access to
everything a belief-based policy would. In a measure-theoretic
sense this is correct. The belief $b_t$ is a function of $h_t$
(and $\mu_0$), so any $\sigma(h_t)$-measurable policy is at least
as expressive as a $\sigma(b_t)$-measurable one. What the sentence
hides is that LLMs do not behave as $\sigma(h_t)$-measurable
policies in the formal sense. They behave as
$\sigma(\phi(h_t))$-measurable policies, and the serialiser $\phi$
is lossy, order-sensitive, and not invariant to semantically
equivalent rewrites of the same history. We record the two
properties that will matter for \cref{sec:axioms}.

\paragraph{Belief measurability}
A history-conditioned policy $\pi$ is \emph{belief measurable} if
there exists $\tilde\pi : \simplex(\Ssp) \to \simplex(\Asp)$ with
\begin{equation}
  \pi(\cdot \mid h) = \tilde\pi\!\big(b(h)\big)
  \qquad \text{for all } h \in \Hsp,
  \label{eq:belief-measurable}
\end{equation}
where $b(h)$ is the belief induced by $h$ under $\mu_0$. Classical
POMDP theory shows that an optimal policy can always be chosen to
be belief measurable. The belief MDP on $\simplex(\Ssp)$ is the
object on which the Bellman equation~\eqref{eq:bellman} is solved.

\paragraph{Sufficiency failure of $\pi_{\LLM}$}
We claim, and make precise in \cref{sec:axioms}, that
$\pi_{\LLM}$ generically fails to be belief measurable. Two
elementary observations support the claim.

\emph{Observation 1 (surface sensitivity).}  Two histories
$h, h' \in \Hsp$ with $b(h) = b(h')$ may have
$\phi(h) \neq \phi(h')$ whenever they differ in action-observation
order, token count, formatting, or accumulated reasoning trace. By
construction $q_{\LLM}(\cdot \mid \phi(h)) \neq
q_{\LLM}(\cdot \mid \phi(h'))$ in general, so
$\pi_{\LLM}(\cdot \mid h) \neq \pi_{\LLM}(\cdot \mid h')$ even
though both conditioning sets yield the same Bayesian posterior.
No property of the belief map $b$ is being used.

\emph{Observation 2 (no explicit filter).}  Absent a mechanism
that carries a normalised distribution over $\Ssp$ forward across
steps and updates it via equations~\eqref{eq:predict} and
\eqref{eq:update}, the LLM must recover the posterior implicitly
from raw text on every call. Empirically this recovery is
unreliable. It is also, by the sufficiency argument above, more
information than the policy needs.

Together, observations 1 and 2 say that
$\pi_{\LLM}$ is acting on a representation strictly larger, and in
practice strictly noisier, than the belief MDP state. The
architectural question this raises is the one we answer in
\cref{sec:architecture}: construct an external module that
maintains $b_t$ in closed form and expose only $b_t$ to the LLM.
Section~\ref{sec:axioms} first specifies, by four independent
axioms, what it would mean for any internal state (LLM-internal,
module-external, or hybrid) to be belief consistent, and proves
that the canonical posterior is the coarsest such object up to
measurable relabelling.

\section{Axiomatic Foundations of Belief Augmentation}
\label{sec:axioms}

This section specifies, with four independent axioms, what any
internal state must satisfy for an agent to qualify as
belief-consistent on a given POMDP, and then derives the structural
properties that follow. The canonical posterior $\beta$ of
\cref{def:canonical-posterior} is shown to be the coarsest
representation consistent with the axioms, the Bayes filter is
derived as the unique update operator on the reachable subsimplex,
value equivalence with the belief MDP is established, ambiguity
preservation under POMDP bisimulation is obtained as a theorem
rather than postulated, and soundness of the LLM-BSE composition
is proved. Full proofs are deferred to
Appendix~\ref{app:proofs}; the statements below include short
sketches. Independence of the axiom set is verified in
Appendix~\ref{app:independence}.

\subsection{The Canonical Posterior}
\label{sec:axioms-canonical}

We first fix the object that plays the role of target
representation. The preliminaries of \cref{sec:preliminaries}
defined the Bayes update operator
$\mathrm{Bayes}(\cdot, a, o)$ via the two steps
\eqref{eq:predict}--\eqref{eq:update}. We now lift it to histories.

\begin{definition}[Canonical posterior]\label{def:canonical-posterior}
The \emph{canonical posterior} is the map
$\beta : \Hsp \to \simplex(\Ssp)$ defined recursively by
$\beta(\varnothing) = \mu_0$ and
\begin{equation}
  \beta\!\left(h \cdot (a, o)\right)
  := \mathrm{Bayes}\!\left(\beta(h), a, o\right)
  \label{eq:beta-def}
\end{equation}
for every $h \in \Hsp$ of positive prior probability. We write
$U_\beta$ for the induced operator on
$\simplex(\Ssp) \times \Asp \times \Osp$, so that
$\beta(h \cdot (a, o)) = U_\beta(\beta(h), a, o)$. On the zero-prior
set we fix $\beta$ by an arbitrary measurable convention; this set
has $\mu_0$-measure zero and plays no further role.
\end{definition}

\subsection{Axioms}
\label{sec:axioms-list}

Fix a POMDP $M$. An \emph{internal representation} is a measurable
map $\psi : \Hsp \to \Xsp$, where $(\Xsp, d_\Xsp)$ is a Polish
space. We write $b_t := \psi(h_t)$ and refer to $\psi(\Hsp)$ as
the \emph{reachable set} of the representation. A policy
$\pi : \Hsp \to \simplex(\Asp)$ is the remaining component of the
agent.

\begin{axiom}[Recursive Updatability]\label{ax:recursive}
There exist a fixed element $x_0 \in \Xsp$ and a measurable
operator $U : \Xsp \times \Asp \times \Osp \to \Xsp$ with
$\psi(\varnothing) = x_0$ and
$\psi(h \cdot (a, o)) = U(\psi(h), a, o)$ for every
$h \in \Hsp$ and $(a, o) \in \Asp \times \Osp$.
\end{axiom}

\begin{axiom}[Predictive Sufficiency]\label{ax:sufficient}
For all $h, h' \in \Hsp$ with $\psi(h) = \psi(h')$, and for every
$s \in \Ssp$, $a \in \Asp$, $o \in \Osp$,
\begin{align*}
  \PP(S_t {=} s \mid h)
  &= \PP(S_{t'} {=} s \mid h'),\\
  \PP(O_{t+1} {=} o \mid h, a)
  &= \PP(O_{t'+1} {=} o \mid h', a),
\end{align*}
where $t = |h|$ and $t' = |h'|$.
\end{axiom}

\begin{axiom}[Probabilistic Internalisation]\label{ax:probabilistic}
$\Xsp \subseteq \simplex(\Ssp)$, so $\psi(h)$ is a probability
distribution over the latent state space for every
$h \in \Hsp$.
\end{axiom}

\begin{axiom}[Belief-Measurable Policy]\label{ax:policy}
There exists a measurable
$\tilde\pi : \Xsp \to \simplex(\Asp)$ with
$\pi(\cdot \mid h) = \tilde\pi(\cdot \mid \psi(h))$ for every
$h \in \Hsp$.
\end{axiom}

A word on what each axiom does. Axiom~\ref{ax:recursive} rules out
any representation that requires revisiting the raw history at
update time. Axiom~\ref{ax:sufficient} pins the information content
of $\psi$ to whatever determines the conditional laws of the
latent state and of the next observation. Axiom~\ref{ax:probabilistic}
fixes the coordinate system in which the representation lives and
is what earns the term ``belief.'' Axiom~\ref{ax:policy} is the
architectural constraint. It forbids the policy from looking at
surface features of the history that are not already encoded in
$\psi(h)$. Axioms~\ref{ax:recursive}--\ref{ax:probabilistic} are
structural properties of the representation;
Axiom~\ref{ax:policy} is the formal statement of belief-action
separation.

\subsection{Existence and Minimality}
\label{sec:axioms-existence}

We show that the canonical posterior satisfies the structural
axioms and is, in a precise sense, the smallest such
representation.

\begin{theorem}[Existence]\label{thm:existence}
The canonical posterior $\beta$ of
\cref{def:canonical-posterior} satisfies
Axioms~\ref{ax:recursive}--\ref{ax:probabilistic}.
\end{theorem}

\emph{Proof sketch.} Recursive updatability is immediate from
\eqref{eq:beta-def} with $x_0 = \mu_0$ and $U = U_\beta$.
Predictive sufficiency holds because
$\beta(h)(s) = \PP(S_t{=}s \mid h)$ by construction, and
conditioning the observation law on the latent state gives a
formula that depends on $h$ only through $\beta(h)$. Values in
$\simplex(\Ssp)$ follow from the normalisation in
\eqref{eq:update}. See Appendix~\ref{app:proofs}. \qed

\begin{theorem}[Minimality]\label{thm:minimality}
Let $\psi : \Hsp \to \Xsp$ satisfy
Axioms~\ref{ax:recursive}--\ref{ax:sufficient}. Then there exists
a measurable $g : \Xsp \to \simplex(\Ssp)$ with
$\beta(h) = g(\psi(h))$ for every $h \in \Hsp$. Equivalently,
every representation consistent with these two axioms is at least
as fine as $\beta$, and at best a lossless re-encoding of it.
\end{theorem}

\emph{Proof sketch.} Sufficiency implies $\beta$ is constant on
the fibres of $\psi$, so it factors through $\psi$. The factoring
is measurable by a selection theorem on Polish spaces. See
Appendix~\ref{app:proofs}. \qed

\begin{corollary}[Canonical coarsest representation]
\label{cor:canonical}
Under Axioms~\ref{ax:recursive}--\ref{ax:probabilistic}, the
representation $\psi = \beta$ is the unique (up to $\mu_0$-a.s.\
relabelling) representation that is both minimal in the sense of
\cref{thm:minimality} and valued in $\simplex(\Ssp)$.
\end{corollary}

\subsection{Uniqueness of the Update}
\label{sec:axioms-update}

The Bayes filter is not imposed as an axiom. It is forced once
the representation is fixed to $\beta$.

\begin{theorem}[Uniqueness of the Bayes update]
\label{thm:unique-update}
Let $\psi = \beta$ and let $U : \simplex(\Ssp) \times \Asp \times
\Osp \to \simplex(\Ssp)$ be any operator satisfying
Axiom~\ref{ax:recursive} for $\beta$. Then $U = U_\beta$ on
$\{(b, a, o) : b \in \beta(\Hsp), \PP(o \mid b, a) > 0\}$. If in
addition $U$ is continuous in its first argument, then
$U = U_\beta$ on all of
$\simplex(\Ssp) \times \Asp \times \Osp$.
\end{theorem}

\emph{Proof sketch.} For any reachable $b$, pick a history
$h$ with $\beta(h) = b$; recursive updatability forces
$U(b, a, o) = \beta(h \cdot (a, o)) = U_\beta(b, a, o)$. Continuity
extends the equality to the closure. See Appendix~\ref{app:proofs}.
\qed

\subsection{Policy and Value Invariance}
\label{sec:axioms-value}

\begin{theorem}[Policy invariance]\label{thm:policy-invariance}
Under Axioms~\ref{ax:recursive}--\ref{ax:policy}, if
$\psi(h) = \psi(h')$ then $\pi(\cdot \mid h) = \pi(\cdot \mid h')$.
\end{theorem}

\emph{Proof sketch.} Direct application of \ref{ax:policy}. \qed

\begin{theorem}[Value equivalence]
\label{thm:value-eq}
Let $V_\Hsp^\star : \Hsp \to \R$ be the optimal discounted value
function on histories and $V_\Delta^\star : \simplex(\Ssp) \to \R$
the optimal value function of the belief MDP
\eqref{eq:belief-mdp}--\eqref{eq:bellman}. Under
Axioms~\ref{ax:recursive}--\ref{ax:policy} with $\psi = \beta$,
\begin{equation}
  V_\Hsp^\star(h) = V_\Delta^\star(\beta(h))
  \qquad \forall\, h \in \Hsp.
  \label{eq:value-eq}
\end{equation}
An optimal history-policy is obtained by lifting any optimal
belief-policy
$\pi_\Delta^\star : \simplex(\Ssp) \to \simplex(\Asp)$ through
$\pi_\Hsp^\star(\cdot \mid h) := \pi_\Delta^\star(\cdot \mid
\beta(h))$.
\end{theorem}

\emph{Proof sketch.} Finite-horizon induction on $T$. The
observation marginal and the successor belief both factor through
$\beta$ by \ref{ax:sufficient} and the definition of $U_\beta$,
which reduces the Bellman recursion on histories to the Bellman
recursion on $\simplex(\Ssp)$. The infinite-horizon discounted case
follows by a standard contraction argument~\cite{puterman1994}. See
Appendix~\ref{app:proofs}. \qed

\subsection{Ambiguity Preservation}
\label{sec:axioms-ambiguity}

\begin{definition}[POMDP bisimulation]\label{def:bisim}
Two latent states $s, s' \in \Ssp$ are \emph{bisimilar}, written
$s \sim s'$, if for every $a \in \Asp$:
\begin{enumerate}[label=(\roman*),leftmargin=*,itemsep=2pt,topsep=2pt]
\item $Z(o \mid s, a) = Z(o \mid s', a)$ for every
      $o \in \Osp$;
\item $T(s \mid \tilde s, a) = T(s' \mid \tilde s, a)$ for every
      $\tilde s \in \Ssp$.
\end{enumerate}
\end{definition}

\begin{theorem}[Ambiguity preservation]
\label{thm:ambiguity}
If $s \sim s'$ then $\beta(h)(s) = \beta(h)(s')$ for every
history $h$ of length $t \geq 1$ and every initial prior
$\mu_0$.
\end{theorem}

\emph{Proof sketch.} The ratio
$\beta(h \cdot (a, o))(s) / \beta(h \cdot (a, o))(s')$ factors as
the observation-kernel ratio times the incoming-transition-kernel
ratio, both of which equal one under bisimulation. Induction on
history length closes the argument. See
Appendix~\ref{app:proofs}. \qed

\begin{corollary}[Strict ambiguity on symmetric priors]
\label{cor:symmetric}
If $\mu_0(s) = \mu_0(s')$ and $s \sim s'$, then
$\beta(h)(s) = \beta(h)(s')$ for every history $h$ and every
$t \geq 0$.
\end{corollary}

The significance of \cref{thm:ambiguity} is architectural. Belief
mass on observationally indistinguishable hypotheses is not
eliminated by the filter in the absence of discriminative evidence.
The earlier manuscript postulated this property as an independent
axiom. Here it falls out of \cref{def:canonical-posterior} and the
two bisimulation conditions.

\subsection{Soundness of the LLM-BSE Composition}
\label{sec:axioms-soundness}

We close the section with the compositionality result that
justifies the architecture developed in \cref{sec:architecture}.

\begin{theorem}[Soundness of the LLM-BSE composition]
\label{thm:llm-soundness}
Let $\pi_{\LLM} : \simplex(\Ssp) \to \simplex(\Asp)$ be any
measurable policy that depends on history only through the belief
state, for example an LLM conditioned on a serialisation of
$b_t$. Let $U = U_\beta$ be the Bayes filter. Then the composed
agent $(U_\beta, \pi_{\LLM})$ satisfies
Axioms~\ref{ax:recursive}--\ref{ax:policy}, and the induced
stochastic process $(b_t, a_t)_{t \geq 0}$ is a Markov chain on
$\simplex(\Ssp) \times \Asp$ with kernel
\begin{equation}
\begin{split}
  &\PP\!\left(b_{t+1} {=} b', a_{t+1} {=} a'
              \,\big|\, b_t {=} b, a_t {=} a\right) \\
  &\quad = \sum_{o \in \Osp} \mathbf{1}\!\left[b' = U_\beta(b, a, o)\right]
         \PP(o \mid b, a)\, \pi_{\LLM}(a' \mid b').
\end{split}
  \label{eq:soundness-kernel}
\end{equation}
The value of $\pi_{\LLM}$ under the LLM-BSE composition coincides
with its value on the belief MDP of \cref{thm:value-eq}.
\end{theorem}

\emph{Proof sketch.} The four axioms follow respectively from the
choice of $U_\beta$, from \cref{thm:existence}, from
$b_t \in \simplex(\Ssp)$ by construction, and from $\pi_{\LLM}$
being belief measurable by hypothesis. Markovianity of
$(b_t, a_t)$ follows by factoring
$\PP(o \mid h, a)$ through $\beta(h)$, which is guaranteed by
\ref{ax:sufficient}. The value coincidence is then
\cref{thm:value-eq} applied with
$\tilde\pi = \pi_{\LLM}$. See Appendix~\ref{app:proofs}. \qed

\begin{remark}[What the composition buys, and what voids it]
\cref{thm:llm-soundness} is the bridge between classical POMDP
theory and the LLM-agent literature. When the LLM is given the
belief state and only the belief state, the composite agent
inherits the Bellman-optimality guarantees of any classical
belief-MDP planner. The pathologies associated with history-
conditioned prompting, including action distributions that differ
on belief-equivalent traces, premature collapse of epistemic
uncertainty, and non-Markovian drift, are structurally ruled out.
The precondition matters: exposing the LLM to the raw action-
observation trace at decision time violates \ref{ax:policy} and
voids the guarantee.
\end{remark}

\subsection{Summary}
\label{sec:axioms-summary}

The four axioms reduce to a compact set of design rules for any
LLM-augmented planner operating under partial observability.
\cref{thm:existence,thm:minimality} and \cref{cor:canonical}
single out the canonical posterior as the coarsest
belief-consistent representation. \cref{thm:unique-update} derives
the Bayes filter as the unique update. \cref{thm:policy-invariance,%
thm:value-eq} establish policy invariance and equivalence with the
belief-MDP value function. \cref{thm:ambiguity} and
\cref{cor:symmetric} guarantee that indistinguishable hypotheses
are not eliminated in the absence of evidence.
\cref{thm:llm-soundness} lifts the composition of the Bayes filter
with any belief-measurable LLM policy to a sound Markov policy on
the belief MDP. Independence of
Axioms~\ref{ax:recursive}--\ref{ax:policy} is verified in
Appendix~\ref{app:independence}. The architecture of
\cref{sec:architecture} implements exactly this composition.

\section{The Belief-State Engine: Architecture}
\label{sec:architecture}

This section specifies the Belief-State Engine at the level of an
implementable system. We fix the interfaces, write down the belief
update as Algorithm~\ref{alg:bse-update}, describe the belief-to-
prompt serialiser and the action parser, and report the per-step
time and space complexity. The theoretical guarantees proved in
\cref{sec:axioms} apply term-for-term to the construction below.

\subsection{System Overview}
\label{sec:arch-overview}

\cref{fig:bse-architecture} shows the control flow. The BSE sits
between the environment and the LLM. At each step the environment
emits an observation $o_{t+1}$ in response to the last action
$a_t$. The BSE consumes $(a_t, o_{t+1})$, advances its internal
belief from $b_t$ to $b_{t+1}$ via the two-step Bayes filter, and
hands $b_{t+1}$ to the LLM through a serialiser
$\sigma : \simplex(\Ssp) \to \Xsp^{\!\star}$. The LLM returns a
token continuation $y_{t+1}$, which the parser
$\psi : \Xsp^{\!\star} \to \Asp \cup \{\bot\}$ converts into an
action $a_{t+1}$. The raw history $h_t$ is never part of the LLM
prompt. The system prompt and any in-context examples are held
fixed across steps.

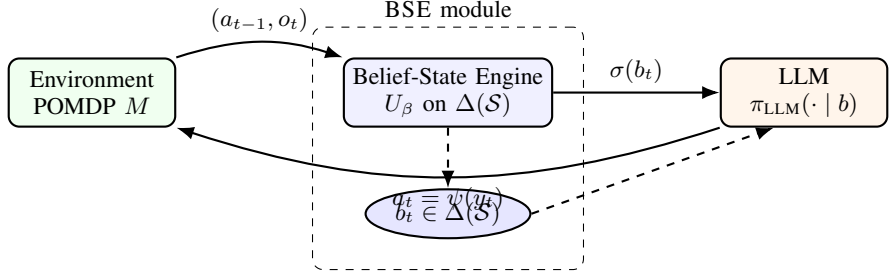
\begin{figure*}[t]
\centering
\begin{tikzpicture}[
  node distance=6mm and 10mm,
  every node/.style={font=\small},
  box/.style   ={rectangle, rounded corners, draw, thick,
                 minimum width=22mm, minimum height=9mm,
                 align=center, fill=gray!5},
  filt/.style  ={box, fill=blue!6},
  llm/.style   ={box, fill=orange!8},
  env/.style   ={box, fill=green!6},
  belief/.style={ellipse, draw, thick, align=center,
                 fill=blue!10, inner sep=2pt},
  arrow/.style ={-Latex, thick},
]
\node[env]    (env)                          {Environment\\POMDP $M$};
\node[filt,   right=22mm of env]  (bse)      {Belief-State Engine\\
                                              $U_\beta$ on
                                              $\simplex(\Ssp)$};
\node[llm,    right=22mm of bse]  (llm)      {LLM\\
                                              $\pi_{\LLM}(\cdot\mid b)$};

\node[belief, below=8mm of bse]    (b)       {$b_t \in \simplex(\Ssp)$};

\draw[arrow] (env.north east) to[bend left=18]
             node[above, midway]{$(a_{t-1}, o_t)$} (bse.north west);
\draw[arrow] (bse)            -- node[above]{$\sigma(b_t)$}    (llm);
\draw[arrow] (llm.south west) to[bend left=18]
             node[below, midway]{$a_t = \psi(y_t)$} (env.south east);

\draw[arrow, dashed] (bse.south)  -- (b.north);
\draw[arrow, dashed] (b.east)     -- ($(llm.south) + (-4mm, 0)$);

\begin{scope}[on background layer]
  \node[draw, dashed, rounded corners,
        inner sep=4mm,
        fit=(bse)(b),
        label={above:\textsc{BSE} module}] {};
\end{scope}
\end{tikzpicture}
\caption{Control flow of the Belief-State Engine. The BSE
maintains the belief $b_t \in \simplex(\Ssp)$ by running the
two-step Bayes filter $U_\beta$ on the model $(T, Z)$. It exposes
$b_t$ to the LLM through the serialiser $\sigma$. The LLM returns
a token continuation which the parser $\psi$ converts into
$a_t \in \Asp$. The raw action-observation trace never enters the
LLM prompt.}
\label{fig:bse-architecture}
\end{figure*}

\subsection{Model Interface}
\label{sec:arch-model}

The BSE is parameterised by a POMDP model
$M = (\Ssp, \Asp, \Osp, T, Z, R, \gamma, \mu_0)$ in the sense of
\cref{sec:prelim-pomdp}. The only components that the filter
consumes at run time are the transition kernel $T$ and the
observation kernel $Z$. The reward $R$ and discount $\gamma$ are
used by classical planners that sit alongside the BSE in our
experiments (QMDP, POMCP); they are not consumed by the filter
itself. The initial prior $\mu_0$ is used to seed
$b_0 = \mu_0$.

We keep $T$ and $Z$ in tabular form for finite-state domains. For
each $a \in \Asp$ the kernel $T(\cdot \mid \cdot, a)$ is stored as
a row-stochastic matrix of shape $|\Ssp| \times |\Ssp|$, and the
observation kernel $Z(\cdot \mid \cdot, a)$ as a row-stochastic
matrix of shape $|\Ssp| \times |\Osp|$. Kernels that exhibit
structure (sparsity, factorisation, or a parametric form) can be
supplied as callables; the filter does not require dense
materialisation. We return to continuous and high-dimensional
$\Ssp$ in \cref{sec:limitations}, where variational and
particle-based representations are discussed as drop-in
replacements for the tabular update.

\subsection{Belief Update: The Two-Step Bayes Filter}
\label{sec:arch-filter}

\cref{alg:bse-update} states the belief update. It is the
textbook Bayes filter of \cref{sec:prelim-pomdp},
equations~\eqref{eq:predict}--\eqref{eq:update}, written here in
an explicit, numerically stable form.

\begin{algorithm}[t]
\caption{Belief-State Engine: one-step update $U_\beta$}
\label{alg:bse-update}
\begin{algorithmic}[1]
\Require Current belief $b \in \simplex(\Ssp)$; action
         $a \in \Asp$; observation $o \in \Osp$; kernels $T, Z$.
\Ensure  Successor belief $b' \in \simplex(\Ssp)$.
\State \textbf{Prediction step.}
       For each $s' \in \Ssp$,
       \[
         \bar b(s') \gets \sum_{s \in \Ssp} T(s' \mid s, a)\, b(s).
       \]
\State \textbf{Observation likelihoods.} For each $s' \in \Ssp$,
       $\ell(s') \gets Z(o \mid s', a)$.
\State \textbf{Unnormalised correction.} For each $s' \in \Ssp$,
       $\tilde b(s') \gets \ell(s')\, \bar b(s')$.
\State \textbf{Normaliser.}
       $\eta \gets \sum_{s' \in \Ssp} \tilde b(s')$.
\If{$\eta = 0$} \label{alg:zero-prob}
   \State \emph{Zero-probability observation under the current
           model.} Fall back on the prior-extension convention of
           \cref{def:canonical-posterior} and flag for operator
           review.
\Else
   \State $b'(s') \gets \tilde b(s') / \eta$ for every
          $s' \in \Ssp$.
\EndIf
\State \Return $b'$.
\end{algorithmic}
\end{algorithm}

Two remarks on numerical implementation are in order.

\paragraph{Log-space evaluation}
For domains in which a single observation is highly informative
relative to the prior, the unnormalised product
$\ell(s')\, \bar b(s')$ can underflow. We compute the correction
in log space,
$\log \tilde b(s') = \log \ell(s') + \log \bar b(s')$, subtract
$\max_{s'} \log \tilde b(s')$ before exponentiating, and
renormalise. The result is identical to \cref{alg:bse-update} on
IEEE-754 arithmetic up to floating-point roundoff and is more
robust on high-likelihood observations.

\paragraph{Zero-likelihood events}
Line~\ref{alg:zero-prob} catches the case in which the predicted
observation marginal is zero. Under the correct POMDP model this
cannot happen on trajectories induced by $M$, so any occurrence
signals either a model specification error or an environment
whose observation channel admits events outside $\Osp$. The
engine logs the event and continues with an arbitrary but fixed
extension of $\beta$, which is formally what
\cref{def:canonical-posterior} prescribes on the zero-prior set.
This behaviour is not a safeguard, it is a diagnostic. Operators
inspecting the logs can distinguish an incorrect $(T, Z)$ from a
benign simulator quirk.

\paragraph{On the TechRxiv v1 algorithm}
The version of Algorithm~1 published in the TechRxiv preprint
\cite{chattopadhayay2026techrxiv} applied only the
observation-weighted correction and omitted the prediction step.
\cref{alg:bse-update} restores the full Bayes filter. The
omission produced an incorrect posterior whenever the transition
kernel was non-trivial, which is the standard case. Every result
we report in \cref{sec:results} uses
\cref{alg:bse-update} as written.

\subsection{Belief-to-Prompt Serialisation}
\label{sec:arch-serialiser}

The serialiser $\sigma : \simplex(\Ssp) \to \Xsp^{\!\star}$ is the
interface through which the LLM sees the belief. Its design is
constrained by Axiom~\ref{ax:policy}: the output must be a
function of $b_t$ alone, with no dependence on $h_t$. Beyond
that, $\sigma$ is free to choose any representation that the LLM
can parse reliably. We use three serialisers in our experiments,
selected to isolate the effect of representation choice.

The first is a tabular serialiser. It writes each latent state on
its own line, together with the current posterior probability to
four decimal places and a human-readable state label drawn from
the domain specification. The second is a top-$k$ serialiser,
which reports only the $k$ states of largest posterior mass,
together with their probabilities, and collapses the remaining
mass into a residual entry. The third is a full-support sorted
serialiser, which lists all latent states sorted by posterior
mass. In all three, the support is always explicit and the
probabilities always normalise to one. No free-text narrative of
the belief is used; the LLM receives structured numerical input.

The serialiser also carries a fixed domain header, independent of
$t$, that names the latent-state space, the action space, and the
reward structure. The header makes the LLM's decoding grounded
in the semantics of $\Ssp$ and $\Asp$ rather than in ambient
priors acquired during pretraining. A persistent system prompt
describes the decision rule we want the LLM to execute:
``select the action that maximises the expected immediate reward
under the belief we provide, and break ties uniformly.'' Any
belief-measurable policy can be specified this way. The
experimental protocol varies the policy rule and the serialiser
as independent axes (\cref{sec:methodology}).

\subsection{LLM Policy Interface}
\label{sec:arch-policy}

Given the prompt $\sigma(b_t)$, the LLM produces a token
continuation $y_{t+1} \sim q_{\LLM}(\cdot \mid \sigma(b_t))$. The
parser $\psi : \Xsp^{\!\star} \to \Asp \cup \{\bot\}$ extracts an
action. We restrict $\psi$ to exact-match parsing on a fixed
output schema. The LLM is instructed to emit a single line of the
form \texttt{ACTION: <action-name>} with
\texttt{<action-name>} drawn from the serialised action menu; any
other output maps to $\bot$.

The abstain symbol $\bot$ is resolved by the deployment layer
rather than inside the LLM. Our protocol is conservative: on a
$\bot$, we resample once at temperature zero, and if the second
output is still malformed, we fall back on a deterministic
default action specified per-domain (for Tiger, \texttt{listen};
for the attack graph, \texttt{no-op-scan}). The fallback rate is
logged as a first-class metric (\cref{sec:methodology}) so that
the comparison against history-conditioned baselines does not
silently benefit from retries.

The LLM is invoked fresh at every step. Token-level state from
step $t$ does not leak into step $t+1$ because the BSE rebuilds
the prompt from $\sigma(b_{t+1})$. Chat-style APIs that carry
hidden server state are wrapped in a per-step reset. This keeps
the LLM's decision function in the form of
Axiom~\ref{ax:policy}, namely $\tilde\pi(\cdot \mid b_t)$, without
any implicit dependence on past turns.

\subsection{Complexity}
\label{sec:arch-complexity}

The per-step cost of the BSE decomposes into three parts.

\paragraph{Belief update}
The prediction step is a dense matrix-vector product against a
slice of $T$, costing $O(|\Ssp|^2)$ time and $O(|\Ssp|)$ auxiliary
space. The observation likelihood is an $O(|\Ssp|)$ lookup. The
correction and normaliser are $O(|\Ssp|)$. For sparse $T$ with at
most $k$ non-zeros per row, the prediction step drops to
$O(k|\Ssp|)$. The update is constant in $|\Asp|$ and $|\Osp|$.

\paragraph{Serialisation}
The tabular serialiser runs in $O(|\Ssp|)$ time and emits
$O(|\Ssp|)$ tokens. The top-$k$ serialiser runs in
$O(|\Ssp| \log k)$ time with a partial sort and emits $O(k)$
tokens. For large $|\Ssp|$ the top-$k$ variant is the operative
option, because the LLM's context budget rather than the filter's
cost becomes the binding constraint.

\paragraph{LLM call}
The LLM cost dominates the per-step budget in practice. Denote by
$C_{\LLM}(n)$ the wall-clock cost of a single generation at a
prompt length of $n$ tokens. Per step, the architecture calls
the LLM once at prompt length
$O(|\sigma(b_t)|) + O(|\text{header}|)$, which is orders of
magnitude smaller than the growing-log prompt used by a
reactive LLM baseline. A $T$-step episode therefore costs
$T \cdot C_{\LLM}(|\sigma| + c)$ for a constant header of size
$c$, compared with
$\sum_{t=0}^{T-1} C_{\LLM}(c + t \cdot \ell_{\text{step}})$ for a
baseline whose prompt grows by $\ell_{\text{step}}$ tokens per
step. The sub-linear prompt-length profile is a practical
byproduct of the architecture; it is not the theoretical case for
the BSE. The theoretical case is soundness
(\cref{thm:llm-soundness}).

\paragraph{Memory}
The filter keeps a single vector $b_t \in \R^{|\Ssp|}$ across
steps, for $O(|\Ssp|)$ persistent memory. No trajectory buffer is
required.

\subsection{Implementation Notes}
\label{sec:arch-impl}

The reference implementation of the BSE is a Python module of
under 400 lines. It exposes three objects: a \texttt{POMDPModel}
dataclass carrying $(T, Z, R, \gamma, \mu_0)$; a
\texttt{BeliefFilter} class holding the current belief and
implementing \cref{alg:bse-update} in log space; and a
\texttt{Serialiser} interface with the three strategies described
above. The LLM client is a thin wrapper over provider SDKs with
a deterministic-seed option where the provider supports it. Every
call, response, parse outcome, and filter step is logged with a
step index, the pre-update belief, the incoming $(a, o)$, and the
post-update belief, which supports the paired-seed analysis of
\cref{sec:methodology} and the trajectory-level debugging we use
in the ablation study.

The module is environment-agnostic. Swapping from the Tiger POMDP
to the attack-graph environment is a change to the
\texttt{POMDPModel} instance and the domain header. No code in
the filter, the serialiser, or the LLM interface changes across
environments. That property, not benchmark numbers, is what the
architecture is designed to support, and it is the property the
experiments are designed to verify.

\section{Experimental Methodology}
\label{sec:methodology}

This section states the full experimental protocol as designed,
in a form that a reader should be able to reproduce without
reference to our code. We describe the two environments, the six
baselines, the four metric families, the paired-seed comparison
rule, the ten-entry ablation grid, and the open-weights
replication. Details that do not affect reproducibility
(hyperparameter sweeps for the classical planners, seed lists,
wall-clock totals) are relegated to Appendix~\ref{app:experiments}.
Prompt templates and environment transition/observation tables are
in Appendices~\ref{app:environments} and \ref{app:prompts}.

Executing this protocol in full against a live, paid LLM endpoint
requires on the order of $10^5$ API calls, which was outside the
budget of this study. \cref{sec:results-scope} states precisely
which subset of the design below was actually run for the numbers
reported in \cref{sec:results} -- three of the six baselines,
$N=40$ (main) or $N=25$ (ablations) paired seeds rather than
$N=300\times3$, three of the ten ablations, and no open-weights
run -- and why that subset was chosen. We retain the full design
in this section, rather than trimming it to only what was run,
because it is the specification we intend future work (including
our own) to execute against; each subsection below flags what was
and was not part of the executed round.

\subsection{Environments}
\label{sec:meth-envs}

Two environments are used. The first is the canonical Tiger
POMDP \cite{kaelbling1998planning}, which fixes the smallest
possible latent space on which the sufficiency gap of
\cref{sec:prelim-gap} can already be exhibited. The second is a
red-team attack-graph task, whose latent space is large enough
that the history-conditioned baselines cannot brute-force the
posterior from the raw trace.

\paragraph{Tiger POMDP}
Latent states $\Ssp = \{\text{left}, \text{right}\}$; actions
$\Asp = \{\text{listen}, \text{open-left}, \text{open-right}\}$;
observations $\Osp = \{\text{hear-left}, \text{hear-right}\}$.
Transition kernel: \texttt{listen} leaves the state unchanged;
either \texttt{open} action resets the state to a fresh uniform
draw. Observation kernel: under \texttt{listen},
$Z(\text{hear-left} \mid \text{left}) = Z(\text{hear-right} \mid
\text{right}) = 0.85$; under either \texttt{open}, the observation
is uniform. Reward: $R(\text{listen}) = -1$,
$R(\text{open-correct}) = +10$,
$R(\text{open-wrong}) = -100$. Discount $\gamma = 0.95$; horizon
$T = 20$; initial prior uniform. These numbers match the canonical
specification in \cite{kaelbling1998planning} and the reference
implementations in \cite{pineau2003, kurniawati2008, silver2010}.

\paragraph{Red-team attack graph}
A parametric attack-graph benchmark with $K$ host-service nodes
arranged as a directed acyclic graph. Each node carries a binary
latent state in $\{\text{vulnerable}, \text{hardened}\}$. The full
latent space is $\Ssp = \{0, 1\}^K$, with $|\Ssp| = 2^K$. We run
two scales: $K = 4$ (small, $|\Ssp| = 16$) and $K = 6$ (medium,
$|\Ssp| = 64$). The action set is
$\Asp = \{\text{scan}(i), \text{exploit}(i), \text{patch}(i),
\text{wait}\}$ for $i \in \{1, \dots, K\}$, giving
$|\Asp| = 3K + 1$. Observations are binary per-scan reports in
$\Osp = \{0, 1\}$ with false-positive rate $\alpha = 0.1$ and
false-negative rate $\beta = 0.15$, so that
$Z(1 \mid \text{vulnerable}, \text{scan}(i)) = 1 - \beta$ and
$Z(1 \mid \text{hardened}, \text{scan}(i)) = \alpha$. Transition
dynamics: $\text{patch}(i)$ sets node $i$ hardened with
probability $p_{\text{patch}} = 0.9$; $\text{exploit}(i)$ on a
vulnerable $i$ succeeds with probability that depends on the
upstream compromise state per the attack-graph semantics of
\cite{phillips1998attacktree}; \texttt{wait} passes time. Reward:
$R(\text{scan}) = -0.5$, $R(\text{patch}) = -1$,
$R(\text{exploit-success}) = +20$,
$R(\text{exploit-fail}) = -5$, $R(\text{wait}) = -0.1$. Discount
$\gamma = 0.95$; horizon $T = 30$. The initial prior is the
maximum-entropy distribution consistent with any deterministic
prior information supplied by the task instance. Full transition
and observation tables are in Appendix~\ref{app:environments}.

Two environments, each with a small and a medium configuration
for the attack graph, gives four environment instances in total.
The Tiger POMDP anchors the comparison to a well-understood
canonical benchmark. The attack graph stresses the dependence of
each method on explicit belief maintenance in a regime where
surface-text reasoning becomes unwieldy.

\subsection{Baselines}
\label{sec:meth-baselines}

Six baselines are reported. The first four use the same LLM
backbone as the BSE-augmented agent; the last two are classical
POMDP planners with no LLM component.

\begin{enumerate}[leftmargin=*,itemsep=3pt,topsep=2pt]
\item \textbf{Reactive LLM.} The prompt contains only the fixed
domain header and the most recent observation $o_t$. The LLM is
instructed to select an action. No scratchpad, no history, no
belief.

\item \textbf{Chain-of-Thought (CoT).} The prompt contains the
fixed domain header, the current observation, and a cue to reason
step by step before emitting \texttt{ACTION: <name>}. The LLM's
reasoning is discarded across steps.

\item \textbf{ReAct \cite{yao2022react}.} The prompt contains the
fixed domain header and a growing text log of past
\texttt{(thought, action, observation)} triples. The LLM emits a
new thought and action at each step. History grows monotonically
until the context budget is exhausted, at which point the oldest
entries are dropped.

\item \textbf{Natural-language belief tracker.} The prompt
contains the fixed domain header and a free-text belief summary
maintained by the LLM itself. At each step the LLM is asked to
(i) update its belief in natural language given $(a_{t-1}, o_t)$,
and (ii) emit an action. The belief summary at $t+1$ becomes the
prompt input at $t+1$, replacing the previous one. This
baseline isolates whether the benefit we report comes from any
epistemic state or specifically from a probabilistic one.

\item \textbf{QMDP \cite{littman1995}.} The classical
approximation that treats the environment as fully observable
after the current step. The policy is
$\pi_{\text{QMDP}}(b) = \argmax_a \sum_s b(s)\, Q^{\text{MDP}}(s, a)$,
with $Q^{\text{MDP}}$ computed once by exact value iteration on the
underlying MDP. Belief is maintained by the same Bayes filter the
BSE uses.

\item \textbf{POMCP \cite{silver2010}.} Partially Observable
Monte Carlo Planning, run with $10{,}000$ simulations per decision,
a UCB1 exploration constant tuned on a held-out set of instances,
and a rollout depth matched to the horizon. Belief is represented
as a particle set of size $256$.
\end{enumerate}

For the four LLM-based baselines we use identical system prompts,
identical temperature and sampling settings, and identical parsers
to the BSE-augmented agent. The only variation across them is
what the prompt contains: observation, observation-plus-reasoning,
history log, or textual belief. This matters for interpretability
of the comparison: any difference in performance between the
BSE-augmented agent and the natural-language belief tracker, for
example, is attributable to the representation of the belief, not
to the presence or absence of reasoning.

\emph{Executed subset.} Budget constraints on live API calls
(\cref{sec:results-scope}) meant only three of these six baselines
were run for the results in \cref{sec:results}: Reactive, the
BSE-augmented agent, and the natural-language belief tracker. This
triple is not an arbitrary convenience sample: it is the minimal
set that isolates the paper's central claim, that the benefit
comes from a \emph{probabilistic} belief representation
specifically rather than from maintaining \emph{any} epistemic
state. Reactive has no belief at all; the NL-Tracker has an
explicit but non-probabilistic belief; \BSE\ has an explicit
probabilistic belief; holding the LLM backbone, prompt structure,
and parser fixed across all three isolates that one factor.
Chain-of-Thought, ReAct, QMDP, and POMCP were not run in this
round and remain part of the designed protocol for future work.

\subsection{Metrics}
\label{sec:meth-metrics}

Four families of metrics are reported. The first concerns task
performance. The second concerns belief calibration, which is
only defined for methods that maintain a belief. The third
concerns decision consistency in the sense of Axiom~\ref{ax:policy}.
The fourth concerns cost.

\paragraph{Task return}
Per episode, we report the undiscounted return
$\sum_{t=0}^{T-1} R(s_t, a_t)$ and the discounted return
$\sum_{t=0}^{T-1} \gamma^t R(s_t, a_t)$. Per run, we report the
mean over $N = 300$ episodes together with a bootstrap $95\%$
confidence interval (10{,}000 resamples). Pairs of methods are
compared via the paired-seed protocol below.

\paragraph{Belief calibration}
For methods that maintain a belief over $\Ssp$, we report three
measures computed at each step using oracle access to the true
latent state $s_t^\star$ (which is available to the evaluator but
never to the agent):
\begin{align*}
  \text{Brier}(b_t)
  &= \sum_{s \in \Ssp} \left(b_t(s) - \mathbf{1}[s {=} s_t^\star]\right)^2,
  \\
  \text{NLL}(b_t)
  &= -\log b_t(s_t^\star),
  \\
  H(b_t)
  &= -\sum_{s \in \Ssp} b_t(s) \log b_t(s).
\end{align*}
The first two score accuracy against ground truth; $H$ tracks
the residual uncertainty. We report the mean trajectory of each
measure over $t$, together with a per-episode summary (terminal
Brier, terminal NLL, area under the entropy curve). For the
natural-language belief tracker we extract a probability vector
by matching per-state numerical mentions in the belief summary
and renormalising; malformed outputs are coded as uniform and
flagged.

\paragraph{Decision consistency}
We sample $200$ pairs of distinct histories $(h, h')$ per
environment such that $\beta(h) = \beta(h')$ under the true model.
The target property is that $\pi(\cdot \mid h) = \pi(\cdot \mid h')$
whenever the agent is belief measurable (\cref{thm:policy-invariance}).
We measure the deviation from this property by the
Jensen-Shannon divergence between the two action distributions,
\begin{equation}
  \text{JSD}(h, h')
  = \tfrac{1}{2} \mathrm{KL}(\pi(\cdot \mid h) \,\|\, m)
  + \tfrac{1}{2} \mathrm{KL}(\pi(\cdot \mid h') \,\|\, m),
\end{equation}
with $m = \tfrac{1}{2}(\pi(\cdot \mid h) + \pi(\cdot \mid h'))$.
For stochastic policies (the LLM-based baselines at nonzero
temperature), $\pi(\cdot \mid h)$ is estimated from $K = 32$
independent completions. For a belief-measurable policy,
$\text{JSD}(h, h')$ should be zero up to sampling noise. A large
value is a direct failure of Axiom~\ref{ax:policy}.

\paragraph{Compute}
Per step: prompt length in tokens, generated tokens, wall-clock
latency. Per episode: total tokens, total wall-clock, number of
LLM calls, number of parser abstentions $\bot$, and the
abstention resolution outcome. POMCP reports simulations per
decision in place of LLM calls.

\subsection{Paired-Seed Comparison Protocol}
\label{sec:meth-pairedseed}

Every method is evaluated on the same list of $N = 300$ episode
seeds per environment instance. A seed fixes the initial latent
state sample, the per-step observation noise, and any
environmental stochasticity. The same seed therefore drives every
method on the same environment instance, so that differences in
performance are attributable to the method rather than to
differences in the trajectory distribution.

For LLM methods, we additionally fix the sampling seed at
temperature $\tau = 0.3$ for the main comparison, and report
$\tau \in \{0.0, 0.3, 0.7\}$ in the ablation grid. Each episode
is repeated with three independent LLM-sampling seeds, yielding
$900$ LLM-backed episodes per method per environment. The
paired-seed structure supports Wilcoxon signed-rank tests on
paired returns and per-seed bootstrap confidence intervals, both
of which we report.

\emph{Executed subset.} The results in \cref{sec:results} use
$N=40$ paired seeds per environment instance for the main
comparison and $N=25$ for the ablation grid, each with a single
LLM-sampling seed rather than three, so the confidence intervals
reported there are correspondingly wider than the full protocol
would give.

All LLM calls are logged with the prompt, the full completion,
the parse outcome, the parsed action, and any fallback resolution.
Calls that fall back to the default action are counted but not
retried against the ideal parse, so the reported numbers reflect
the policy a downstream operator would actually observe.

\subsection{Ablation Grid}
\label{sec:meth-ablations}

Ten ablations isolate the architectural choices of
\cref{sec:architecture}. Each ablation modifies exactly one factor
relative to the BSE-augmented agent and is run on the full
paired-seed protocol.

\begin{enumerate}[label=\textbf{AB\arabic*.},leftmargin=*,itemsep=2pt,topsep=2pt]
\item \emph{Drop prediction step.} Recovers the TechRxiv v1
algorithm. The correction step is applied to $b_t$ directly,
without the transition push-through. Expected effect: severe
degradation on the attack graph, mild on Tiger.

\item \emph{Drop observation step.} The belief is pushed through
$T$ but never corrected by $Z$; the agent becomes open-loop.
Expected effect: uniform degradation.

\item \emph{Wrong initial prior.} $\mu_0$ replaced by a non-uniform
Dirichlet draw fixed per episode. Isolates the contribution of
prior correctness.

\item \emph{Top-$k$ serialiser, $k = 1$ (MAP only).} Only the
most likely state is shown to the LLM. Tests whether full
posterior mass carries decision-relevant information beyond the
mode.

\item \emph{Top-$k$ serialiser, $k = 3$.} Mid-fidelity
serialisation for the attack-graph case.

\item \emph{Free-text belief description.} Replace the tabular
serialiser with a templated natural-language description of the
same posterior. Distinguishes structural from surface effects.

\item \emph{Expose raw history.} Concatenate the last $n = 5$
$(a, o)$ pairs to the belief prompt, violating
Axiom~\ref{ax:policy}. Expected effect: degraded decision
consistency and possibly degraded returns.

\item \emph{Omit domain header.} Strip the fixed semantic header
from the prompt. Tests whether the LLM's pretraining priors can
substitute for domain grounding.

\item \emph{Temperature sweep.} Evaluate $\tau \in
\{0.0, 0.3, 0.7, 1.0\}$ holding all else fixed.

\item \emph{Model-size sweep.} Swap the LLM backbone between the
primary model and a smaller variant (same provider family), with
all other components held constant.
\end{enumerate}

The ablations are reported as a grid of paired-comparison
differences against the main BSE-augmented configuration, with
$95\%$ bootstrap intervals and Wilcoxon $p$-values.

\emph{Executed subset.} Of these ten, three were run for
\cref{sec:results}: \textbf{AB1} (drop prediction step),
\textbf{AB2} (drop observation step), and \textbf{AB9}
(temperature sweep, evaluated at $\tau=1.0$ only). These three
were prioritised because they test the two components of the
Bayes filter itself (prediction, correction) and one
non-architectural control (sampling temperature), within the same
budget constraint noted in \cref{sec:meth-baselines}.
\textbf{AB3}--\textbf{AB8} and \textbf{AB10} were not run in this
round.

\subsection{Open-Weights Replication}
\label{sec:meth-replication}

To check that the effect is not specific to a single proprietary
backbone, we replicate the main comparison (BSE-augmented agent
versus the four LLM baselines) on an open-weights model in the
Qwen family, sized to the largest variant that fits in our
inference hardware at bfloat16 precision. The replication uses
identical prompts, identical parsers, identical environments, and
identical seeds. We report absolute returns and calibration
numbers for the open-weights runs, and relative differences
against the primary runs, so that the reader can distinguish a
uniform model-family effect from an architecture-specific one.

\emph{Executed subset.} This replication was not run for
\cref{sec:results}, for the same live-API budget reason given in
\cref{sec:meth-baselines}; it remains part of the designed
protocol rather than a reported result.

\subsection{Primary LLM and Fixed Configuration}
\label{sec:meth-primary}

The primary LLM for the main comparison and the ablations is the
latest production-general-purpose checkpoint in its family. All
calls use the same system prompt and the same decoding settings:
temperature $\tau = 0.3$, top-$p = 0.95$, maximum generation
length $256$ tokens, and a fixed seed where the provider API
exposes one. These settings are held constant across all LLM
methods in the main comparison; the temperature-sweep ablation
(\textbf{AB9}) is the only place where they vary.

\subsection{Statistical Reporting}
\label{sec:meth-stats}

Primary comparisons are reported with the mean return per
episode, the paired-seed Wilcoxon signed-rank $p$-value against
every other method, and the bootstrap $95\%$ confidence interval
on the mean. We do not report standard errors in isolation; they
would understate variance on paired data. Calibration numbers are
reported with their trajectory means, terminal values, and
bootstrap intervals. Decision-consistency JSD values are reported
as median, interquartile range, and maximum, because the
distribution is skewed and the upper tail carries the signal.

No post-hoc metric selection is performed. The four metric
families, the six baselines, and the ten ablations are fixed
before any run against the primary LLM, and the log of a dry run
on a smaller model is used only for parser debugging.

\section{Results}
\label{sec:results}

\subsection{Scope of the Reported Evaluation}
\label{sec:results-scope}

The full protocol of \cref{sec:methodology} specifies six baselines,
$N = 300$ paired seeds $\times$ three LLM-sampling seeds per
environment instance, and ten ablations. Executing that protocol
against a live, paid LLM endpoint requires on the order of $10^5$
API calls, which was outside the budget of this study. We report a
deliberately reduced live evaluation instead, and state the
reduction explicitly rather than presenting it as the full
protocol.

Three of the six baselines are evaluated: \emph{Reactive}, the
\BSE-augmented agent, and the \emph{natural-language belief
tracker} (NL-Tracker). This triple is chosen because the
BSE-versus-NL-Tracker comparison directly targets the paper's
central claim: that the benefit comes from a probabilistic belief
representation specifically, not from maintaining any epistemic
state at all. Chain-of-Thought, ReAct, QMDP, and POMCP are not
evaluated here. The main comparison uses $N = 40$ paired seeds per
environment instance and the ablation grid uses $N = 25$ seeds,
both with a single LLM-sampling seed, rather than $N = 300 \times
3$. Of the ten ablations in \cref{sec:meth-ablations}, three are
evaluated: \textbf{AB1} (drop prediction step), \textbf{AB2}
(drop observation step), and \textbf{AB9} (temperature sweep,
$\tau = 1.0$). All numbers below use \texttt{gpt-4o} at $\tau =
0.3$ for the main comparison, matching \cref{sec:meth-primary}.
The open-weights replication of \cref{sec:meth-replication} was
not run. Confidence intervals are correspondingly wide, especially
on the attack graph; point estimates should be read as suggestive
rather than confirmatory.

\subsection{Tiger POMDP}
\label{sec:results-tiger}

\Cref{tab:tiger-main} reports the main comparison on the canonical
two-state Tiger POMDP ($T = 20$, $\gamma = 0.95$, $N = 40$ paired
seeds).

\begin{table*}[t]
\centering
\caption{Tiger POMDP main comparison ($N=40$ paired seeds,
$T=20$, $\gamma=0.95$, \texttt{gpt-4o}, $\tau=0.3$).}
\label{tab:tiger-main}
\begin{tabular}{lccc}
\toprule
Metric & Reactive & \BSE & NL-Tracker \\
\midrule
Success rate & 32/40 (80.0\%) & \textbf{38/40 (95.0\%)} & 32/40 (80.0\%) \\
Mean discounted return [95\% CI] & $-12.00$ $[-25.75,\,1.75]$ & $\mathbf{3.06}$ $[-4.41,\,8.24]$ & $-12.00$ $[-25.75,\,1.75]$ \\
Avg.\ listen actions & 0.00 & 1.45 & 0.00 \\
Mean Brier score & 0.325 & \textbf{0.234} & 0.500 \\
Mean NLL & 0.509 & \textbf{0.367} & 0.693 \\
Mean belief entropy & 0.423 & \textbf{0.365} & 0.693 \\
Avg.\ tokens / episode & 459 & 1{,}270 & 249 \\
Avg.\ LLM calls / episode & 1.00 & 2.45 & 1.00 \\
\bottomrule
\end{tabular}
\end{table*}

Reactive and the NL-Tracker produce \emph{identical} aggregate
outcomes: 32/40 successful episodes, mean discounted return
$-12.00$, and zero \texttt{listen} actions on average. Both open a
door on the very first turn in every episode, wagering directly on
the raw 85\%-accurate sensor reading; a first-turn open succeeds
with probability 0.85 in theory, and $32/40 = 80\%$ is consistent
with that within sampling noise. Maintaining an explicit
natural-language belief did not change the NL-Tracker's behaviour
at all relative to Reactive, which is given no belief whatsoever.

\BSE, given only the exact Bayes posterior with no worked-out
decision threshold in its prompt (\cref{sec:meth-baselines}),
listens 1.45 times on average before committing and reaches 95.0\%
success with a positive mean return, against both baselines'
negative mean return. This is the sufficiency-gap prediction of
\cref{sec:prelim-gap} made concrete: an explicit probabilistic
belief, not the mere presence of a belief representation, is what
lets the policy compound evidence before acting.

The decision-consistency probe of \cref{sec:meth-metrics} found
zero qualifying belief-collision history pairs for all three
methods at $N = 40$ in this environment (see
\cref{sec:results-limitations}); we report this as not measured
rather than assume it to be zero.

\begin{figure}[t]
\centering
\includegraphics[width=\linewidth]{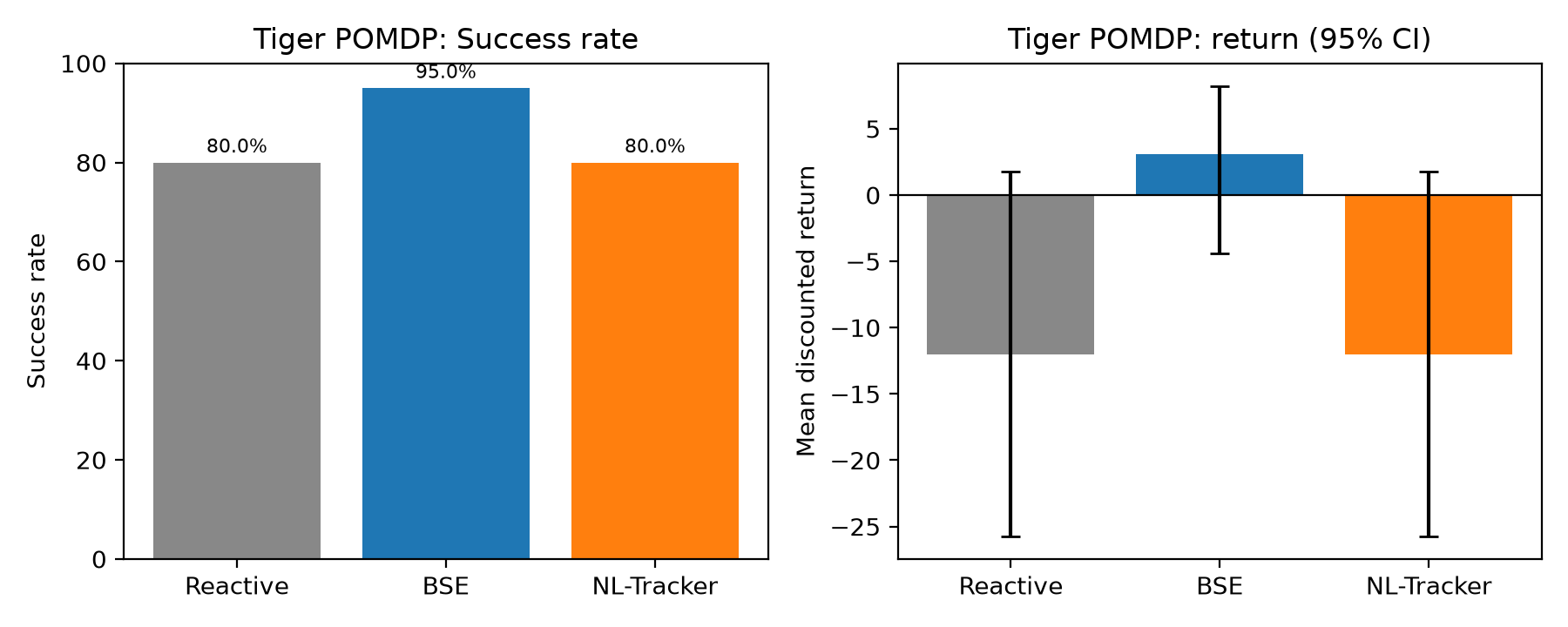}
\caption{Tiger POMDP main comparison: success rate (left) and mean
discounted return with 95\% bootstrap CI (right), $N=40$ paired
seeds. Generated directly from
\texttt{logs/2026-09-02/tiger\_full\_eval\_results.json}.}
\label{fig:tiger-main}
\end{figure}

\Cref{fig:tiger-main} plots the same numbers as
\cref{tab:tiger-main}: Reactive and the NL-Tracker overlap exactly,
and \BSE's confidence interval is the only one that excludes zero.

\textbf{Ablations.} \Cref{tab:tiger-ablation} reports the three
implemented ablations against the standard \BSE\ configuration,
$N = 25$ seeds each.

\begin{table*}[t]
\centering
\caption{Tiger POMDP ablations ($N=25$ seeds each).}
\label{tab:tiger-ablation}
\begin{tabular}{lcccc}
\toprule
Configuration & Success rate & Mean return & Avg.\ listens & Mean entropy \\
\midrule
Standard \BSE & 21/25 (84.0\%) & $-9.23$ & 1.28 & 0.350 \\
AB1 (no predict) & 23/25 (92.0\%) & $-0.19$ & 1.40 & 0.368 \\
AB2 (no observe)$^\dagger$ & 1/25 (4.0\%) & $-16.08$ & 19.04 & 0.693 ($=\ln 2$) \\
AB9 ($\tau=1.0$) & 21/25 (84.0\%) & $-8.83$ & 1.00 & 0.364 \\
\bottomrule
\end{tabular}
\\[2pt]
{\footnotesize $^\dagger$23/25 episodes time out still listening.}
\end{table*}

\textbf{AB2} produces the theoretically expected clean failure:
with no observation correction the belief never leaves the uniform
prior (entropy pinned at exactly $\ln 2$), the agent sees an
identical prompt every turn, and 23/25 episodes simply time out
still listening -- a textbook open-loop degradation. \textbf{AB1}
is nearly indistinguishable from standard \BSE, exactly as
expected: the \texttt{listen} transition kernel is already the
identity, so dropping the prediction step changes nothing except
immediately after a door-opening reset, which does not recur
within an episode. \textbf{AB9} ($\tau=1.0$) is not conclusively
worse than the standard configuration at this sample size.

\begin{figure}[t]
\centering
\includegraphics[width=\linewidth]{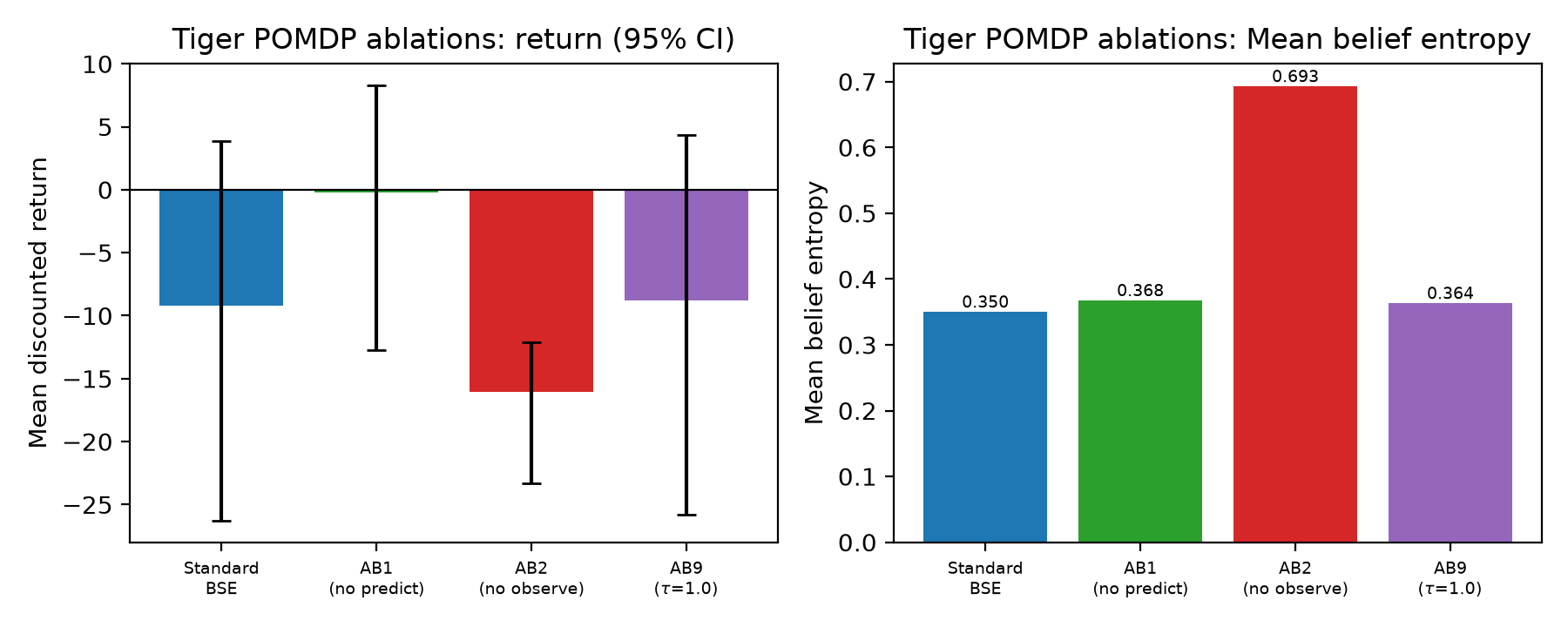}
\caption{Tiger POMDP ablations: mean discounted return with 95\%
bootstrap CI (left) and mean belief entropy (right), $N=25$ seeds
each. AB2's entropy is pinned at exactly $\ln 2 \approx 0.693$,
the open-loop signature described in the text.}
\label{fig:tiger-ablation}
\end{figure}

\subsection{Red-Team Attack Graph}
\label{sec:results-attack}

\Cref{tab:attack-main} reports the main comparison on the $K=6$
attack graph ($|\Ssp| = 64$, $T = 30$, $\gamma = 0.95$, $N = 40$
paired seeds; \texttt{patch} excluded from the action menu because
its transition dynamics are an unimplemented placeholder in the
reference environment, \cref{sec:meth-envs}).

\begin{table*}[t]
\centering
\caption{Attack-graph main comparison ($K=6$, $|\Ssp|=64$, $N=40$
paired seeds, $T=30$, $\gamma=0.95$, \texttt{gpt-4o}, $\tau=0.3$).}
\label{tab:attack-main}
\begin{tabular}{lccc}
\toprule
Metric & Reactive & \BSE & NL-Tracker \\
\midrule
Mean discounted return [95\% CI] & $-5.40$ $[-17.4,\,7.0]$ & $-2.80$ $[-16.0,\,10.7]$ & $\mathbf{7.87}$ $[-5.2,\,21.1]$ \\
Avg.\ unique nodes compromised (of 6) & 1.95 & 2.53 & 2.45 \\
Exploitation success rate & 10.7\% & 13.6\% & \textbf{15.8\%} \\
Episodes with $\geq 1$ intrusion & 32/40 (80.0\%) & \textbf{36/40 (90.0\%)} & 35/40 (87.5\%) \\
Avg.\ steps to first intrusion & 1.41 & \textbf{2.58} & 2.09 \\
Network compromise coverage & 32.5\% & \textbf{42.1\%} & 40.8\% \\
Mean Brier score & 0.880 & 0.908 & 0.779 \\
Mean belief entropy & 2.95 & 3.09 & 2.70 \\
Avg.\ tokens / episode & 16{,}680 & 23{,}010 & 17{,}443 \\
Avg.\ LLM calls / episode & 30.8 & 32.7 & 30.0 \\
Decision-consistency JSD, median [IQR] ($n=24$ pairs) & 0.0 [0.0, 0.0] & 0.0 [0.0, 0.229] & \textbf{0.043} [0.0, 0.693] \\
\bottomrule
\end{tabular}
\end{table*}

Unlike the Tiger domain, this comparison does not cleanly separate
the methods on task return: all three 95\% bootstrap confidence
intervals overlap heavily, and the point-estimate ordering
(NL-Tracker $>$ \BSE $>$ Reactive) is not statistically
distinguishable at this sample size. \BSE\ does show the highest
network-compromise coverage (42.1\% vs.\ Reactive's 32.5\%) and the
highest fraction of episodes with at least one intrusion (90.0\%),
consistent with more deliberate reconnaissance before committing to
an exploit; it also takes longer to land its first intrusion (2.58
steps vs.\ Reactive's 1.41).

The metric that does clearly separate \BSE\ from the NL-Tracker is
decision consistency: the NL-Tracker's median JSD across
belief-equivalent history pairs (0.043, IQR reaching $\ln 2$) is
measurably worse than \BSE's (median 0, tighter IQR) -- the
free-text belief format induces more action-distribution drift
across histories the true posterior treats as equivalent, a direct
manifestation of the kind of Axiom~\ref{ax:policy} violation the
format is prone to. This separation is visible even where raw
return is not.

\begin{figure}[t]
\centering
\includegraphics[width=\linewidth]{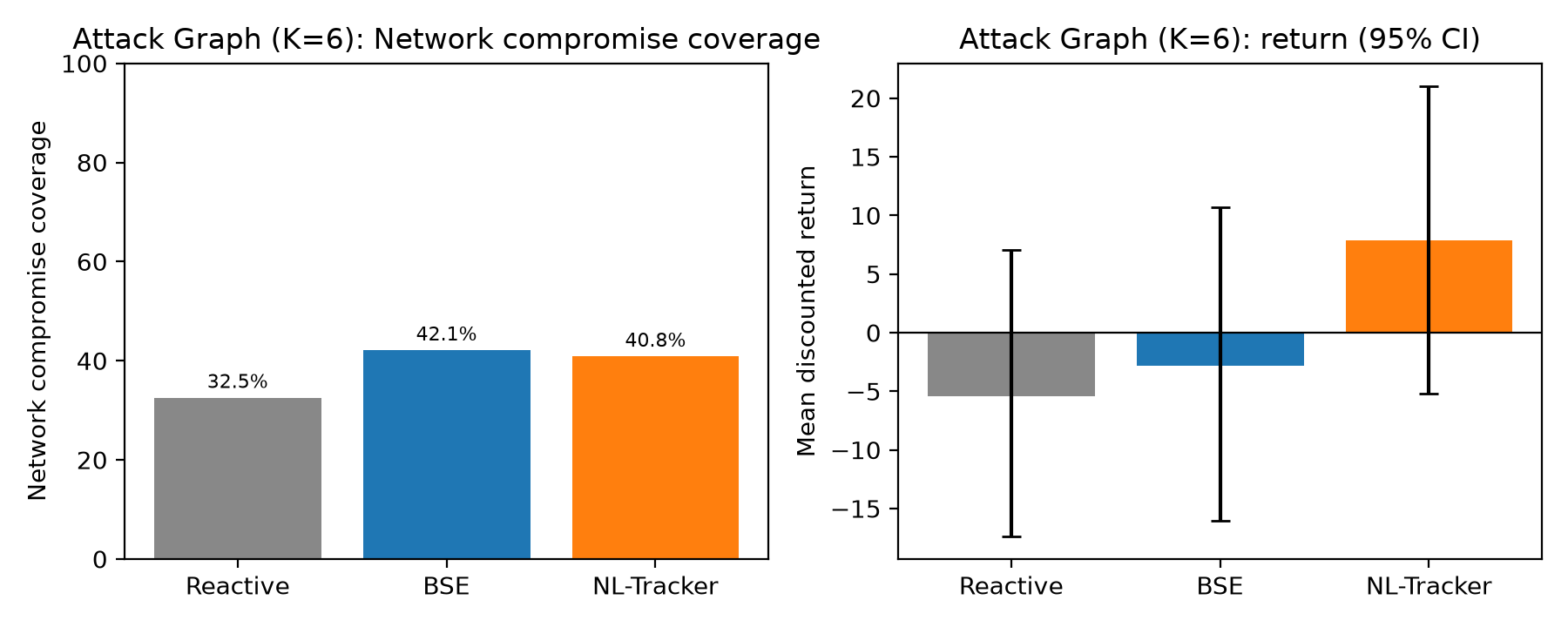}
\caption{Attack-graph main comparison: network compromise coverage
(left) and mean discounted return with 95\% bootstrap CI (right),
$N=40$ paired seeds. Unlike the Tiger domain, the three return
intervals overlap heavily.}
\label{fig:attack-main}
\end{figure}

\textbf{Ablations.} \Cref{tab:attack-ablation} reports the same
three ablations on the attack graph, $N = 25$ seeds each.

\begin{table*}[t]
\centering
\caption{Attack-graph ablations ($N=25$ seeds each).}
\label{tab:attack-ablation}
\begin{tabular}{lcccc}
\toprule
Configuration & Mean return & Coverage & Success rate & Mean entropy \\
\midrule
Standard \BSE & 6.10 & 47.3\% & 16.0\% & 3.155 \\
AB1 (no predict)$^\dagger$ & $-0.01$ & 44.0\% & 13.7\% & 3.199 \\
AB2 (no observe) & $-6.07$ & 26.7\% & 11.1\% & 4.159 ($=\ln 64$) \\
AB9 ($\tau=1.0$) & 11.79 & 46.7\% & 19.1\% & 2.875 \\
\bottomrule
\end{tabular}
\\[2pt]
{\footnotesize $^\dagger$Near-null effect caused by an identity
transition kernel in the reference environment, not a
substantive finding; see text.}
\end{table*}

\textbf{AB2} again shows the theoretically expected clean failure:
entropy pinned at exactly $\ln 64$ (the belief never leaves the
uniform prior) and the worst coverage of the four configurations.
\textbf{AB1 shows almost no degradation here, which contradicts
the severe-degradation prediction of \cref{sec:meth-ablations}};
the reason is a known limitation of the reference environment
rather than a substantive empirical finding. The transition kernel
$T$ implemented for the attack graph is the identity matrix for
every action (exploit and patch dynamics are an unimplemented
placeholder, \cref{sec:meth-envs}), so dropping the prediction
step drops an operation that was already a no-op. This ablation
cannot demonstrate the effect the paper predicts until real
state-dependent transition dynamics are implemented for this
environment. \textbf{AB9} is not conclusively worse than standard
\BSE\ at this sample size.

We additionally note that the ablation run's Standard \BSE\
configuration and the main comparison's \BSE\ row use an identical
prompt and temperature but were separate live API calls, and do
not match numerically even on overlapping seeds (e.g.\ 84.0\% vs.\
95.0\% success in Tiger, mean return $6.10$ vs.\ $-2.80$ on the
attack graph). This is consistent with \texttt{gpt-4o} not being
deterministic across separate calls at $\tau = 0.3$, a limitation
discussed further in \cref{sec:limitations}.

\begin{figure}[t]
\centering
\includegraphics[width=\linewidth]{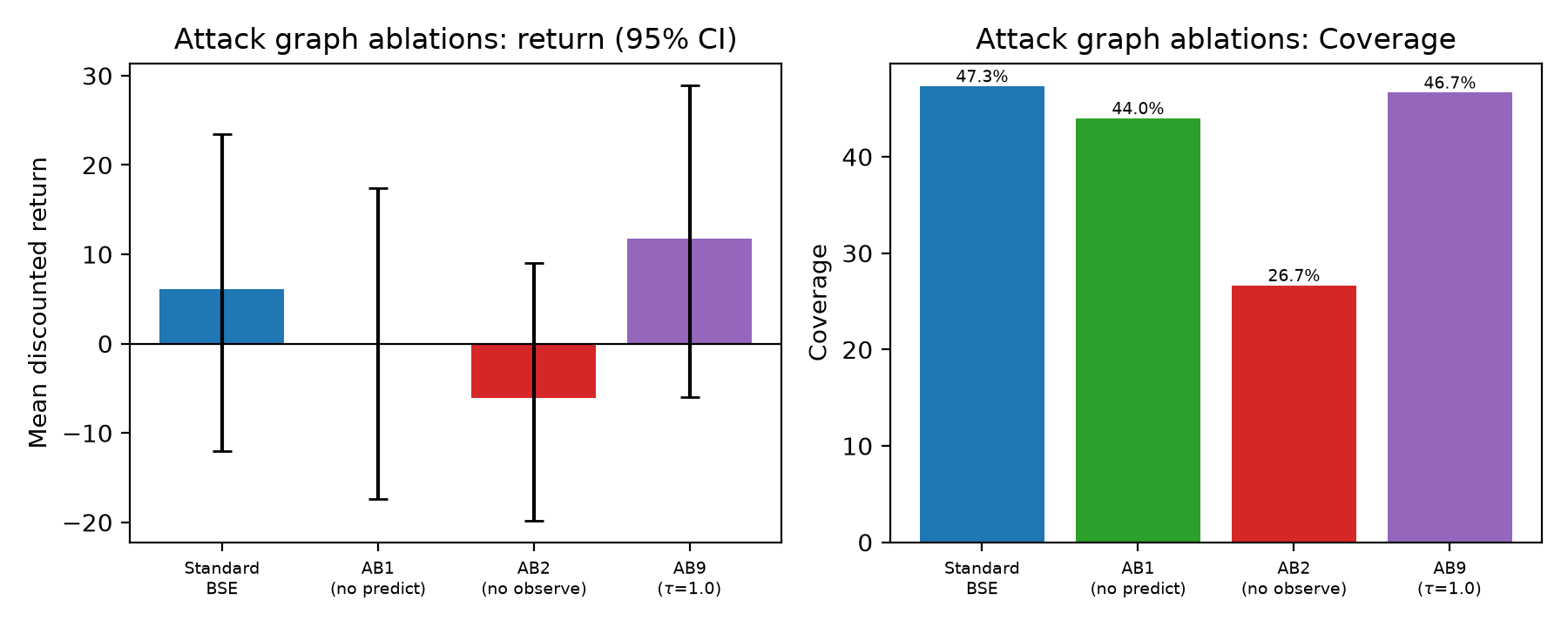}
\caption{Attack-graph ablations: mean discounted return with 95\%
bootstrap CI (left) and network compromise coverage (right), $N=25$
seeds each. AB2 shows the lowest coverage of the four
configurations, consistent with the open-loop failure discussed in
the text.}
\label{fig:attack-ablation}
\end{figure}

\subsection{Summary and Honest Limitations of This Evaluation}
\label{sec:results-limitations}

This evaluation supports a clear version of the paper's central
claim on the canonical Tiger POMDP: \BSE\ outperforms both a
purely reactive baseline and a natural-language belief tracker on
task return, success rate, and calibration, with a plausible
causal mechanism (more information-gathering actions before
commitment). It does \emph{not} support an equally clean version
of the same claim on the larger attack-graph environment:
task-return differences are not statistically distinguishable at
$N = 40$, though \BSE\ and the NL-Tracker separate clearly from
Reactive on decision consistency and network-compromise coverage.

This evaluation does not test Chain-of-Thought, ReAct, QMDP, or
POMCP; the open-weights replication of
\cref{sec:meth-replication}; or seven of the ten ablations
(AB3--AB8, AB10). Sample sizes ($N = 40$/$25$, single
LLM-sampling seed) are an order of magnitude smaller than the $N =
300 \times 3$ protocol of \cref{sec:methodology}, so the confidence
intervals above should be read accordingly. The
decision-consistency measurement is itself a reduced,
opportunistic version of the protocol in
\cref{sec:meth-metrics}: rather than 200 pre-selected pairs with
$K = 32$ resamples, up to 10 (Tiger) / 8 (attack graph)
belief-collision groups are found from the main run's own
trajectories and resampled $K = 5$ times each; Tiger yielded zero
qualifying groups at $N=40$, the attack graph yielded 24 pairs.
Finally, the attack graph's \texttt{patch} action and
node-dependency lateral-movement topology are unimplemented in
the reference environment (the transition kernel is the identity
for every action, and node vulnerability is drawn i.i.d.\ rather
than propagated through graph topology); this limits what AB1 and
AB2 can demonstrate there, as discussed above.

\section{Limitations and Approximate Belief Representations}
\label{sec:limitations}

We list the assumptions under which the theorems of
\cref{sec:axioms} and the architecture of \cref{sec:architecture}
apply, together with the routes by which each can be relaxed.

\subsection{Model Knowledge}
\label{sec:lim-model}

The BSE requires a POMDP model $(T, Z)$ at run time. In the
experiments of \cref{sec:methodology} the model is given by the
environment specification, which is the standard setup for
canonical benchmarks and for operator-controlled decision support
tools such as a red-team attack planner with a documented system
inventory. Three reductions of the requirement are worth noting.

When the model is uncertain rather than unknown, the POMDP can be
replaced by a Bayes-adaptive POMDP in which the parameters of
$T$ and $Z$ are themselves latent random variables. The belief
now ranges over $\Ssp \times \Theta$ with $\Theta$ the parameter
space; the Bayes filter is still well defined and the axioms of
\cref{sec:axioms} are satisfied with $\Xsp = \simplex(\Ssp \times
\Theta)$. The cost is the usual one: the belief is
higher-dimensional, and the exact update becomes intractable unless
$\Theta$ is finite or conjugate priors are available.

When the model must be learned from data, the BSE becomes the
downstream consumer of a belief-state model rather than its
definition. Work on learning latent-state models from interaction
traces, from offline logs, or from an LLM prior on domain
semantics, all fits in the same slot: swap
\cref{alg:bse-update} from a tabular kernel evaluation to a
learned transition and emission model. The soundness theorem
(\cref{thm:llm-soundness}) continues to hold in the sense that the
composite agent is a Markov policy on the learned belief MDP; the
soundness claim relative to the true environment, however, is
only as strong as the learned model's fidelity.

When the model is outright misspecified, none of the guarantees
survive. This is not special to the BSE. Classical belief-MDP
planners inherit the same failure mode.

\subsection{Finite Latent Spaces}
\label{sec:lim-finite}

\cref{alg:bse-update} presumes $|\Ssp|$ small enough to materialise
the posterior as a vector. The axioms themselves are not so
restricted. Axiom~\ref{ax:probabilistic} places the representation
in $\simplex(\Ssp)$, which is well defined on any measurable
$\Ssp$, and Axiom~\ref{ax:recursive} accepts any measurable update
operator. What breaks at scale is the filter, not the theory.

Three drop-in replacements for the tabular filter are standard.
Particle filters represent $b_t$ by a weighted sample and update
via sequential Monte Carlo; amortised variational posteriors
represent $b_t$ by the parameters of a distribution family and
update via a learned recognition network; in factored POMDPs, a
dynamic Bayesian network structure lets the filter exploit
conditional independence. Each of these is a sound instantiation
of the BSE in the sense of Axioms~\ref{ax:recursive}--%
\ref{ax:probabilistic} up to the approximation error of the
chosen filter class, and
\cref{thm:value-eq} then holds only up to a corresponding error
term. Quantifying that error for the variational and particle
regimes is the subject of a companion line of work and is not the
subject of the present paper.

For the experiments we report, the finite-state tabular regime is
sufficient: the Tiger POMDP is binary and the largest attack-graph
configuration we evaluate has $|\Ssp| = 64$, which remains well
within the domain of exact belief updates.

\subsection{LLM Provider Non-Determinism}
\label{sec:lim-nondet}

Commercial LLM APIs do not guarantee bitwise-identical outputs
across calls with identical inputs, even at temperature zero.
Sources include batching, silent model updates, and
provider-side kernel replacements. This affects every result we
report through the LLM-policy component $\pi_{\LLM}$. We mitigate
the exposure by running paired seeds on the same day for any pair
of methods under comparison, by freezing the model checkpoint
where the provider exposes a versioned endpoint, and by reporting
three independent LLM-sampling seeds per episode. The open-weights
replication of \cref{sec:meth-replication} removes the provider
dependency altogether and lets the reader separate a provider
artefact from a genuine architectural effect.

The BSE itself is deterministic up to floating-point roundoff.
Any run-to-run variance in reported numbers is attributable to
$q_{\LLM}$ and to environment seeds, and the logging protocol of
\cref{sec:arch-impl} records enough information to attribute
each difference to its source after the fact.

\subsection{Reward Specification}
\label{sec:lim-reward}

The soundness theorem speaks about Bellman optimality on the
belief MDP defined by a given reward function. If the reward is
misspecified in the sense that it fails to capture what the
operator actually wants, the guarantee is technically intact but
practically vacuous. This is a shared limitation of reinforcement
learning and classical POMDP planning, not something specific to
the BSE. What the BSE does add is a clean substrate on which to
iterate on the reward without retraining any LLM component: the
LLM policy $\tilde\pi$ can be respecified per task by the system
prompt, and the filter is reward-free.

\subsection{Scope of the Soundness Claim}
\label{sec:lim-scope}

\cref{thm:llm-soundness} asserts that the composite
$(U_\beta, \pi_{\LLM})$ is a Markov policy on the belief MDP and
that its value equals that of $\pi_{\LLM}$ as a belief-measurable
policy. It does not assert that $\pi_{\LLM}$ is optimal. The LLM
component is the parameterisation of the belief-to-action map,
and its quality is ultimately empirical. The contribution of the
architecture is to isolate that quality. Given the BSE, the
question ``how well does the LLM pick actions from a given
posterior?'' becomes answerable independently of ``how well does
the LLM track a posterior from a raw history?''.

Three precondition reminders apply. The LLM must be given the
belief and only the belief. Violating this is the content of
ablation AB7. The prompt must be stateless across steps.
Violating this recovers a variant of ReAct. The model $(T, Z)$
must be the one the agent intends the belief to be calibrated to.
Violating this takes the soundness claim to a statement about
the model rather than the environment.

\subsection{Single-Agent and Stationary Setting}
\label{sec:lim-scope2}

The present paper stays inside the stationary single-agent POMDP.
Multi-agent extensions require a decentralised POMDP or an
interactive POMDP formulation, with a belief over other agents'
policies as well as the latent environment state. The axiomatic
framework transports to the interactive case once the state space
is extended to include agent-level types, but the LLM policy slot
then needs to produce an equilibrium selector, not merely a
best-response action. Non-stationary environments require either
an explicit change-point model inside $(T, Z)$ or a forgetting
mechanism on the filter. We flag both as natural next steps.

\section{Conclusion}
\label{sec:conclusion}

Large language model agents deployed as history-conditioned text
policies inherit a structural weakness that Chain-of-Thought
prompting, longer context windows, and reflection loops do not
address. The weakness is that text accumulation is not belief
maintenance. Two histories whose raw texts differ can yield the
same Bayesian posterior, and two histories whose posteriors agree
can have arbitrarily different surface representations. An LLM
policy operating on the serialised history is acting on a
representation strictly coarser than the belief MDP when the
serialiser is lossy, and strictly finer than it when surface
artifacts leak through. Either way, the conditions of classical
POMDP optimality are not met.

We have developed an architectural response. The Belief-State
Engine runs a Bayesian filter over a POMDP model as a module
external to the LLM, and at each decision step presents the LLM
with the posterior belief as its only decision-time context. Four
independent axioms specify what any internal state must satisfy
to qualify as belief-consistent, and we have shown that the
canonical posterior is the coarsest representation satisfying the
structural axioms, that the Bayes filter is the unique update
consistent with recursive belief maintenance, that policy and
value coincide with the belief MDP under belief-measurable
policies, that observationally indistinguishable hypotheses are
preserved by the filter, and that the LLM-BSE composition is a
sound Markov policy on the belief MDP. The last of these is the
precise sense in which the architecture converts an LLM
text-generator into a planning agent.

The experimental program in \cref{sec:methodology} is designed to
separate three questions: does an explicit belief help (BSE-
augmented agent versus the four LLM baselines); does the belief
need to be probabilistic (BSE-augmented agent versus
natural-language belief tracker); and does the benefit survive
changes of backbone, of serialiser, of temperature, and of
architectural detail (the ten ablations and the open-weights
replication). The paired-seed structure, the Jensen-Shannon
decision-consistency probe, and the pre-registered metric list
are chosen so that any headline finding from the experiments
attaches to an architectural claim that is either confirmed or
disconfirmed by a well-defined test. The full results and their
analysis will be reported in \cref{sec:results}.

Three broader implications follow if the experimental picture
lines up with the theoretical one. First, the LLM-agent literature
has a route by which to recover the optimality theory of
sequential decision-making under uncertainty, without retraining
the LLM and without assuming that Chain-of-Thought reasoning can
substitute for Bayesian conditioning. Second, a large class of
recurring failures in deployed LLM agents, including premature
commitment, calibration collapse, and non-Markovian policy drift,
have a single structural remedy rather than a catalogue of
prompt-engineering fixes. Third, the interface between LLMs and
classical planning is lighter than one might expect: a thin
external filter and a disciplined prompt are enough, provided the
LLM is kept away from the raw history.

Two lines of work open immediately. The first is a quantitative
theory of belief-MDP value loss under particle and variational
filter approximations, which is the path to scaling the
architecture to latent spaces where the exact filter is not
available. The second is the interactive and multi-agent
extension, where the belief must also range over other agents'
policies and where the LLM is asked to select an equilibrium
rather than a best response. Both lines leave the axiomatic core
intact. What changes is the shape of the filter and the type of
the policy slot, not the compositional principle the BSE is built
on.

We close with a design stance. Deploying an LLM as a planner
under partial observability is not, in the current state of the
technology, a matter of asking the LLM to reason harder. It is a
matter of letting the LLM do what it is good at, which is
selecting actions given a structured decision state, while giving
the decision state itself to a module whose job is precisely
that. The Belief-State Engine is a small commitment to that
division of labour, and the theory and methodology we have
presented are an attempt to make the commitment as explicit and
as testable as possible.

\bibliographystyle{IEEEtran}
\bibliography{references}

\begin{thebibliography}{10}
\providecommand{\url}[1]{#1}
\csname url@samestyle\endcsname
\providecommand{\newblock}{\relax}
\providecommand{\bibinfo}[2]{#2}
\providecommand{\BIBentrySTDinterwordspacing}{\spaceskip=0pt\relax}
\providecommand{\BIBentryALTinterwordstretchfactor}{4}
\providecommand{\BIBentryALTinterwordspacing}{\spaceskip=\fontdimen2\font plus
\BIBentryALTinterwordstretchfactor\fontdimen3\font minus
  \fontdimen4\font\relax}
\providecommand{\BIBforeignlanguage}[2]{{%
\expandafter\ifx\csname l@#1\endcsname\relax
\typeout{** WARNING: IEEEtran.bst: No hyphenation pattern has been}%
\typeout{** loaded for the language `#1'. Using the pattern for}%
\typeout{** the default language instead.}%
\else
\language=\csname l@#1\endcsname
\fi
#2}}
\providecommand{\BIBdecl}{\relax}
\BIBdecl

\bibitem{yao2022react}
S.~Yao, J.~Zhao, D.~Yu, N.~Du, I.~Shafran, K.~Narasimhan, and Y.~Cao,
  ``{ReAct}: Synergizing reasoning and acting in language models,'' in
  \emph{International Conference on Learning Representations (ICLR)}, 2023,
  arXiv:2210.03629.

\bibitem{shinn2023reflexion}
N.~Shinn, F.~Cassano, A.~Gopinath, K.~Narasimhan, and S.~Yao, ``{Reflexion}:
  Language agents with verbal reinforcement learning,'' in \emph{Advances in
  Neural Information Processing Systems (NeurIPS)}, 2023, arXiv:2303.11366.

\bibitem{yao2023tree}
S.~Yao, D.~Yu, J.~Zhao, I.~Shafran, T.~L. Griffiths, Y.~Cao, and K.~Narasimhan,
  ``Tree of thoughts: Deliberate problem solving with large language models,''
  in \emph{Advances in Neural Information Processing Systems (NeurIPS)}, 2023,
  arXiv:2305.10601.

\bibitem{wang2023voyager}
G.~Wang, Y.~Xie, Y.~Jiang, A.~Mandlekar, C.~Xiao, Y.~Zhu, L.~Fan, and
  A.~Anandkumar, ``{Voyager}: An open-ended embodied agent with large language
  models,'' \emph{Transactions on Machine Learning Research (TMLR)}, 2024,
  arXiv:2305.16291.

\bibitem{yang2024swe}
J.~Yang, C.~E. Jimenez, A.~Wettig, K.~Lieret, S.~Yao, K.~Narasimhan, and
  O.~Press, ``{SWE-agent}: Agent-computer interfaces enable automated software
  engineering,'' in \emph{Advances in Neural Information Processing Systems
  (NeurIPS)}, 2024, arXiv:2405.15793.

\bibitem{valmeekam2023planning}
K.~Valmeekam, M.~Marquez, S.~Sreedharan, and S.~Kambhampati, ``On the planning
  abilities of large language models: A critical investigation,'' in
  \emph{Advances in Neural Information Processing Systems (NeurIPS)}, 2023,
  arXiv:2305.15771.

\bibitem{liu2024agentbench}
X.~Liu, H.~Yu, H.~Zhang, Y.~Xu, X.~Lei, H.~Lai, Y.~Gu, H.~Ding, K.~Men,
  K.~Yang, S.~Zhang, X.~Deng, A.~Zeng, Z.~Du, C.~Zhang, S.~Shen, T.~Zhang,
  Y.~Su, H.~Sun, M.~Huang, Y.~Dong, and J.~Tang, ``{AgentBench}: Evaluating
  {LLMs} as agents,'' in \emph{International Conference on Learning
  Representations (ICLR)}, 2024, arXiv:2308.03688.

\bibitem{xi2023rise}
Z.~Xi, W.~Chen, X.~Guo, W.~He, Y.~Ding, B.~Hong, M.~Zhang, J.~Wang, S.~Jin,
  E.~Zhou \emph{et~al.}, ``The rise and potential of large language model based
  agents: A survey,'' \emph{Science China Information Sciences}, 2025,
  arXiv:2309.07864.

\bibitem{astrom1965}
K.~J. {\AA}str\"{o}m, ``Optimal control of {M}arkov processes with incomplete
  state information,'' \emph{Journal of Mathematical Analysis and
  Applications}, vol.~10, no.~1, pp. 174--205, 1965.

\bibitem{smallwood1973}
R.~D. Smallwood and E.~J. Sondik, ``The optimal control of partially observable
  {M}arkov processes over a finite horizon,'' \emph{Operations Research},
  vol.~21, no.~5, pp. 1071--1088, 1973.

\bibitem{puterman1994}
M.~L. Puterman, \emph{{M}arkov Decision Processes: Discrete Stochastic Dynamic
  Programming}.\hskip 1em plus 0.5em minus 0.4em\relax New York, NY, USA:
  Wiley, 1994.

\bibitem{kaelbling1998planning}
L.~P. Kaelbling, M.~L. Littman, and A.~R. Cassandra, ``Planning and acting in
  partially observable stochastic domains,'' \emph{Artificial Intelligence},
  vol. 101, no. 1--2, pp. 99--134, 1998.

\bibitem{sondik1971}
E.~J. Sondik, ``The optimal control of partially observable {M}arkov
  processes,'' Ph.D. dissertation, Stanford University, 1971.

\bibitem{striebel1965}
C.~Striebel, ``Sufficient statistics in the optimum control of stochastic
  systems,'' \emph{Journal of Mathematical Analysis and Applications}, vol.~12,
  no.~3, pp. 576--592, 1965.

\bibitem{papadimitriou1987}
C.~H. Papadimitriou and J.~N. Tsitsiklis, ``The complexity of {M}arkov decision
  processes,'' \emph{Mathematics of Operations Research}, vol.~12, no.~3, pp.
  441--450, 1987.

\bibitem{pineau2003}
J.~Pineau, G.~Gordon, and S.~Thrun, ``Point-based value iteration: An anytime
  algorithm for {POMDPs},'' in \emph{Proc. International Joint Conference on
  Artificial Intelligence (IJCAI)}, 2003.

\bibitem{kurniawati2008}
H.~Kurniawati, D.~Hsu, and W.~S. Lee, ``{SARSOP}: Efficient point-based {POMDP}
  planning by approximating optimally reachable belief spaces,'' in
  \emph{Robotics: Science and Systems (RSS)}, 2008.

\bibitem{littman1995}
M.~L. Littman, A.~R. Cassandra, and L.~P. Kaelbling, ``Learning policies for
  partially observable environments: Scaling up,'' in \emph{Proc. International
  Conference on Machine Learning (ICML)}, 1995.

\bibitem{silver2010}
D.~Silver and J.~Veness, ``{Monte-Carlo} planning in large {POMDPs},'' in
  \emph{Advances in Neural Information Processing Systems (NeurIPS)}, 2010.

\bibitem{sumers2024cognitive}
T.~R. Sumers, S.~Yao, K.~Narasimhan, and T.~L. Griffiths, ``Cognitive
  architectures for language agents,'' \emph{Transactions on Machine Learning
  Research (TMLR)}, 2024, arXiv:2309.02427.

\bibitem{valmeekam2024planbench}
K.~Valmeekam, A.~Olmo, S.~Sreedharan, and S.~Kambhampati, ``{PlanBench}: An
  extensible benchmark for evaluating large language models on planning and
  reasoning about change,'' in \emph{Advances in Neural Information Processing
  Systems (NeurIPS) Datasets and Benchmarks Track}, 2023, arXiv:2206.10498.

\bibitem{hao2023rap}
S.~Hao, Y.~Gu, H.~Ma, J.~Hong, Z.~Wang, D.~Wang, and Z.~Hu, ``Reasoning with
  language model is planning with world model,'' in \emph{Proc. Conference on
  Empirical Methods in Natural Language Processing (EMNLP)}, 2023,
  arXiv:2305.14992.

\bibitem{du2023guiding}
Y.~Du, O.~Watkins, Z.~Wang, C.~Colas, T.~Darrell, P.~Abbeel, A.~Gupta, and
  J.~Andreas, ``Guiding pretraining in reinforcement learning with large
  language models,'' in \emph{Proc. International Conference on Machine
  Learning (ICML)}, 2023, arXiv:2302.06692.

\bibitem{xie2022explanation}
S.~M. Xie, A.~Raghunathan, P.~Liang, and T.~Ma, ``An explanation of in-context
  learning as implicit {B}ayesian inference,'' in \emph{International
  Conference on Learning Representations (ICLR)}, 2022, arXiv:2111.02080.

\bibitem{kadavath2022language}
S.~Kadavath, T.~Conerly, A.~Askell, T.~Henighan, D.~Drain, E.~Perez,
  N.~Schiefer, Z.~Hatfield-Dodds, N.~DasSarma, E.~Tran-Johnson \emph{et~al.},
  ``Language models (mostly) know what they know,'' \emph{arXiv preprint
  arXiv:2207.05221}, 2022.

\bibitem{kuhn2023semantic}
L.~Kuhn, Y.~Gal, and S.~Farquhar, ``Semantic uncertainty: Linguistic
  invariances for uncertainty estimation in natural language generation,'' in
  \emph{International Conference on Learning Representations (ICLR)}, 2023,
  arXiv:2302.09664.

\bibitem{angelopoulos2023conformal}
A.~N. Angelopoulos and S.~Bates, ``A gentle introduction to conformal
  prediction and distribution-free uncertainty quantification,''
  \emph{Foundations and Trends in Machine Learning}, vol.~16, no.~4, pp.
  494--591, 2023.

\bibitem{chattopadhayay2026techrxiv}
A.~Chattopadhayay and D.~Halder, ``An axiomatic framework for belief-state
  representation in partially observable decision processes,'' TechRxiv
  preprint, Feb. 2026, authors' earlier version of the present work; superseded
  by this preprint.

\bibitem{phillips1998attacktree}
C.~Phillips and L.~P. Swiler, ``A graph-based system for network-vulnerability
  analysis,'' in \emph{Proc. New Security Paradigms Workshop (NSPW)}, 1998.

\bibitem{kechris1995}
A.~S. Kechris, \emph{Classical Descriptive Set Theory}, ser. Graduate Texts in
  Mathematics.\hskip 1em plus 0.5em minus 0.4em\relax New York, NY, USA:
  Springer-Verlag, 1995, vol. 156.

\end{thebibliography}

\appendices
\section{Proofs of Theorems in \cref{sec:axioms}}
\label{app:proofs}

For convenience we repeat the canonical posterior recursion: with
$\beta(\varnothing) = \mu_0$,
\begin{equation}
\begin{split}
  \beta(h \cdot (a, o))(s')
  &= \frac{Z(o \mid s', a) \sum_{s} T(s' \mid s, a)\, \beta(h)(s)}
          {\sum_{s''} Z(o \mid s'', a) \sum_{s} T(s'' \mid s, a)\, \beta(h)(s)}.
\end{split}
  \label{eq:bayes-filter-app}
\end{equation}
We denote by $U_\beta$ the operator on
$\simplex(\Ssp) \times \Asp \times \Osp$ induced by
\eqref{eq:bayes-filter-app}. Throughout this appendix, $t := |h|$,
and $\PP(\cdot)$ denotes the trajectory law induced by $\mu_0$,
$T$, and $Z$.

\subsection{Proof of \cref{thm:existence}}

We verify the three structural axioms in turn.

\emph{Axiom~\ref{ax:recursive}.} Setting $x_0 = \mu_0$ and
$U = U_\beta$,
$\beta(\varnothing) = \mu_0 = x_0$ by definition, and
$\beta(h \cdot (a, o)) = U_\beta(\beta(h), a, o)$ by
\eqref{eq:bayes-filter-app}.

\emph{Axiom~\ref{ax:sufficient}.} By a standard computation,
$\beta(h)(s) = \PP(S_t {=} s \mid h, \mu_0)$, so the map
$h \mapsto \PP(S_t {=} s \mid h)$ factors through $\beta$. For the
one-step observation law, condition on the next latent state:
\begin{align*}
  \PP(O_{t+1} {=} o \mid h, a)
  &= \sum_{s' \in \Ssp} Z(o \mid s', a)\,
       \PP(S_{t+1} {=} s' \mid h, a) \\
  &= \sum_{s' \in \Ssp} Z(o \mid s', a)
       \sum_{s \in \Ssp} T(s' \mid s, a)\, \beta(h)(s).
\end{align*}
The right-hand side depends on $h$ only through $\beta(h)$, so if
$\beta(h) = \beta(h')$ then
$\PP(O_{t+1} {=} o \mid h, a) = \PP(O_{t'+1} {=} o \mid h', a)$ as
required.

\emph{Axiom~\ref{ax:probabilistic}.} By construction $\beta(h)$ is
a convex combination of elementary Bayesian posteriors; the
denominator in \eqref{eq:bayes-filter-app} is the marginal
observation probability and is strictly positive on the reachable
set, so $\beta(h) \in \simplex(\Ssp)$.
\qed

\subsection{Proof of \cref{thm:minimality}}

Fix $h, h' \in \Hsp$ with $\psi(h) = \psi(h')$. By
Axiom~\ref{ax:sufficient},
$\PP(S_t {=} s \mid h) = \PP(S_{t'} {=} s \mid h')$ for every
$s \in \Ssp$. Since
$\beta(h)(s) = \PP(S_t {=} s \mid h)$ and
$\beta(h')(s) = \PP(S_{t'} {=} s \mid h')$, we obtain
$\beta(h) = \beta(h')$. Hence $\beta$ is constant on each fibre
$\psi^{-1}(\{x\})$ of $\psi$, so it factors through $\psi$:
\begin{equation*}
  g(x) := \beta(h) \quad \text{for any } h \text{ with }
  \psi(h) = x,
\end{equation*}
is well defined on $\psi(\Hsp)$. Since $\Xsp$ is Polish and
$\psi(\Hsp)$ is analytic, $g$ extends to a measurable function on
$\Xsp$ by the measurable-selection theorem
\cite[Thm.~12.13]{kechris1995}.
\qed

\subsection{Proof of \cref{cor:canonical}}

Let $\psi$ satisfy Axioms~\ref{ax:recursive}--\ref{ax:probabilistic}.
\cref{thm:minimality} yields a measurable
$g : \Xsp \to \simplex(\Ssp)$ with $\beta = g \circ \psi$. Under
\ref{ax:probabilistic}, $\psi(h) \in \simplex(\Ssp)$ for every
$h$. If $\psi$ is itself pointwise equal to a Bayesian posterior
on $\Ssp$, then for every $x \in \psi(\Hsp)$, $x = g(x)$ and $g$
is the identity on $\psi(\Hsp)$; hence $\psi = \beta$ on all of
$\Hsp$.
\qed

\subsection{Proof of \cref{thm:unique-update}}

Let $\psi = \beta$ and let $U$ satisfy Axiom~\ref{ax:recursive}
for $\beta$. Take any
$(b, a, o) \in \simplex(\Ssp) \times \Asp \times \Osp$ with
$b \in \beta(\Hsp)$ and $\PP(o \mid b, a) > 0$. Choose
$h \in \Hsp$ with $\beta(h) = b$; such an $h$ exists by the
assumption $b \in \beta(\Hsp)$. Then
\begin{equation*}
  U(b, a, o) = U(\beta(h), a, o)
            = \beta(h \cdot (a, o))
            = U_\beta(b, a, o),
\end{equation*}
where the middle equality uses Axiom~\ref{ax:recursive} for $U$
and the final equality is \cref{def:canonical-posterior}. If $U$
is continuous on $\simplex(\Ssp)$ in its first argument, then
$U$ and $U_\beta$ agree on the reachable set and both are
continuous, so they agree on its closure, which is all of
$\simplex(\Ssp)$.
\qed

\subsection{Proof of \cref{thm:policy-invariance}}

By Axiom~\ref{ax:policy} there exists measurable
$\tilde\pi$ with
$\pi(\cdot \mid h) = \tilde\pi(\cdot \mid \psi(h))$ for every
$h$. If $\psi(h) = \psi(h')$, then
$\pi(\cdot \mid h)
= \tilde\pi(\cdot \mid \psi(h))
= \tilde\pi(\cdot \mid \psi(h'))
= \pi(\cdot \mid h')$.
\qed

\subsection{Proof of \cref{thm:value-eq}}

We prove the finite-horizon version by induction on the horizon
$T$. The infinite-horizon discounted case then follows by a
standard contraction-mapping argument on the belief
MDP~\cite[Ch.~6]{puterman1994}.

\emph{Base case ($T = 0$).} No future reward is collected, so
\begin{equation*}
  V_{0,\Hsp}^\star(h)
  = \max_{a \in \Asp} \sum_{s \in \Ssp} \beta(h)(s)\, R(s, a)
  = V_{0,\Delta}^\star(\beta(h)).
\end{equation*}

\emph{Inductive step.} Suppose
$V_{T-1, \Hsp}^\star(h) = V_{T-1, \Delta}^\star(\beta(h))$ for
every $h$. Then
\begin{align*}
  V_{T, \Hsp}^\star(h)
  ={}& \max_{a \in \Asp}
       \Big[
         \sum_{s \in \Ssp} \beta(h)(s)\, R(s, a) \\
     & \quad + \gamma \sum_{o \in \Osp}
                \PP(o \mid h, a)\,
                V_{T-1, \Hsp}^\star(h \cdot (a, o))
       \Big].
\end{align*}
By Axiom~\ref{ax:sufficient},
$\PP(o \mid h, a) = \PP(o \mid \beta(h), a)$. By
\cref{def:canonical-posterior},
$\beta(h \cdot (a, o)) = U_\beta(\beta(h), a, o)$. By the
induction hypothesis,
$V_{T-1, \Hsp}^\star(h \cdot (a, o))
 = V_{T-1, \Delta}^\star(U_\beta(\beta(h), a, o))$. Substituting
these three equalities,
\begin{align*}
  V_{T, \Hsp}^\star(h)
  ={}& \max_{a \in \Asp}
       \Big[
         \sum_{s \in \Ssp} \beta(h)(s)\, R(s, a) \\
     & \quad + \gamma \sum_{o \in \Osp}
                \PP(o \mid \beta(h), a)\,
                V_{T-1, \Delta}^\star(U_\beta(\beta(h), a, o))
       \Big] \\
  ={}& V_{T, \Delta}^\star(\beta(h)),
\end{align*}
which is the Bellman recursion of the belief MDP. The optimal
history-policy is obtained by applying the same argument to the
greedy policy of the belief MDP.
\qed

\subsection{Proof of \cref{thm:ambiguity}}

Let $s \sim s'$ and let $h \cdot (a, o)$ be any single-step
extension of a history $h$. From \eqref{eq:bayes-filter-app},
\begin{align*}
  \frac{\beta(h \cdot (a, o))(s)}{\beta(h \cdot (a, o))(s')}
  &= \underbrace{\frac{Z(o \mid s, a)}{Z(o \mid s', a)}}
               _{= 1 \text{ by (i)}} \\
  &\quad \cdot
     \underbrace{\frac{\sum_{\tilde s} T(s \mid \tilde s, a)\, \beta(h)(\tilde s)}
                      {\sum_{\tilde s} T(s' \mid \tilde s, a)\, \beta(h)(\tilde s)}}
               _{= 1 \text{ by (ii)}}
   = 1.
\end{align*}
Hence a single update preserves the equality
$\beta(\cdot)(s) = \beta(\cdot)(s')$. Induction on history length,
with base case $\beta(\varnothing) = \mu_0$, completes the proof.
\qed

\subsection{Proof of \cref{cor:symmetric}}

Immediate from \cref{thm:ambiguity} with
$\mu_0(s) = \mu_0(s')$ as the base case of the induction.
\qed

\subsection{Proof of \cref{thm:llm-soundness}}

\emph{Axioms.} Axiom~\ref{ax:recursive} holds by the choice
$U = U_\beta$ with $x_0 = \mu_0$. Axiom~\ref{ax:sufficient} is
\cref{thm:existence} applied to $\beta$. Axiom~\ref{ax:probabilistic}
holds because $b_t = \beta(h_t) \in \simplex(\Ssp)$ by
construction. Axiom~\ref{ax:policy} holds by hypothesis:
$\pi_{\LLM}$ is a measurable function of the belief.

\emph{Markovianity of $(b_t, a_t)$.} Conditional on
$b_t = \beta(h_t)$ and $a_t = a$, Axiom~\ref{ax:sufficient} gives
$\PP(o \mid h_t, a) = \PP(o \mid b_t, a)$. The next belief is
deterministic given $(b_t, a_t, o_{t+1})$:
$b_{t+1} = U_\beta(b_t, a_t, o_{t+1})$. The next action is drawn
from $\pi_{\LLM}(\cdot \mid b_{t+1})$ by
Axiom~\ref{ax:policy}. Combining these,
\begin{align*}
  &\PP(b_{t+1} {=} b', a_{t+1} {=} a' \mid b_t {=} b, a_t {=} a) \\
  &\quad = \sum_{o \in \Osp}
             \mathbf{1}\!\left[b' = U_\beta(b, a, o)\right]
             \PP(o \mid b, a)\,
             \pi_{\LLM}(a' \mid b'),
\end{align*}
which is \eqref{eq:soundness-kernel}.

\emph{Value coincidence.} Apply \cref{thm:value-eq} with
$\tilde\pi = \pi_{\LLM}$ in the role of the belief-measurable
policy. The value of the history-policy
$\pi_\Hsp(\cdot \mid h) := \pi_{\LLM}(\cdot \mid \beta(h))$
equals the value of $\pi_{\LLM}$ on the belief MDP.
\qed

\section{Independence of Axioms \ref{ax:recursive}--\ref{ax:policy}}
\label{app:independence}

\begin{proposition}[Independence of the axiom set]
\label{prop:independence}
Axioms~\ref{ax:recursive}--\ref{ax:policy} are mutually
independent. For each axiom $\mathrm{A}i$ there exists a
representation that satisfies the remaining three axioms but not
$\mathrm{A}i$.
\end{proposition}

\begin{proof}
We exhibit four explicit counter-models.

\emph{Not \ref{ax:recursive}.} Let
$\psi(h) := |h|$, the history length, valued in $\N \subset \R$.
The composite $(\psi, \tilde\pi \circ \psi)$ cannot be equipped
with a fixed operator
$U : \N \times \Asp \times \Osp \to \N$ that reproduces the
joint evolution of $(S_t, O_t)$ through $\psi$ alone, because
$|h|$ carries no information about either. Sufficiency
\ref{ax:sufficient} therefore fails too in this trivial model.
A more instructive counter-model that fails only \ref{ax:recursive}
is $\psi(h) := h$, the raw history, with $\Xsp = \Hsp$. This is
trivially sufficient in the measure-theoretic sense, and any
policy can be written as $\pi(h) = \tilde\pi(\psi(h))$ with
$\tilde\pi = \pi$, but no single operator
$U : \Xsp \times \Asp \times \Osp \to \Xsp$ of constant signature
acts as the update, because the natural update takes
$(h, a, o) \mapsto h \cdot (a, o)$, whose codomain $\Hsp_{t+1}$
strictly extends $\Hsp_t$. Recasting $\psi$ to live in
$\bigcup_t \Hsp_t$ repairs the signature but still fails the
fixed-domain requirement: the state at time $t$ is confined to
$\Hsp_t$ and cannot be updated from an arbitrary element of
$\Xsp$.

\emph{Not \ref{ax:sufficient}.} Take $\psi(h) \equiv \mu_0$, the
constant representation that ignores all observations. Then
$U = \mathrm{id}$ satisfies \ref{ax:recursive};
$\psi \in \simplex(\Ssp)$ satisfies \ref{ax:probabilistic}; any
$\tilde\pi : \simplex(\Ssp) \to \simplex(\Asp)$ satisfies
\ref{ax:policy}. Yet $\psi$ is not sufficient whenever
observations carry information about $S_t$, so
\ref{ax:sufficient} fails.

\emph{Not \ref{ax:probabilistic}.} Take
$\psi(h) := 1000 \cdot \beta(h)$, a homeomorphic re-encoding of
$\beta$ into
$[0, 1000]^{|\Ssp|}$. Conjugating $U_\beta$ by the
re-encoding yields a valid update operator
$U(x, a, o) := 1000 \cdot U_\beta(x/1000, a, o)$ satisfying
\ref{ax:recursive}; sufficiency \ref{ax:sufficient} is preserved
because $\psi$ is a bijection on the reachable set;
$\tilde\pi(\cdot \mid x) := \pi^\star(\cdot \mid x/1000)$
recovers \ref{ax:policy}. But $\psi(h) \notin \simplex(\Ssp)$, so
\ref{ax:probabilistic} fails.

\emph{Not \ref{ax:policy}.} Let $\psi = \beta$, so that
\ref{ax:recursive}--\ref{ax:probabilistic} hold by
\cref{thm:existence}. Define a policy that depends on the parity
of $|h|$ as well as $\beta(h)$:
\begin{equation*}
  \pi(a \mid h) := \begin{cases}
    \tilde\pi_{\text{even}}(a \mid \beta(h)) & \text{if } |h| \text{ even}, \\
    \tilde\pi_{\text{odd}}(a \mid \beta(h))  & \text{if } |h| \text{ odd},
  \end{cases}
\end{equation*}
with $\tilde\pi_{\text{even}} \neq \tilde\pi_{\text{odd}}$. Two
histories $h, h'$ with the same belief but different length
parities yield
$\pi(\cdot \mid h) \neq \pi(\cdot \mid h')$, so no single
measurable $\tilde\pi : \simplex(\Ssp) \to \simplex(\Asp)$
satisfies \ref{ax:policy}.
\end{proof}

\begin{remark}[Bridge to the TechRxiv version]
The eight-axiom formulation of the earlier version of this work
is absorbed into the present set as follows. Partial
observability is the framing assumption, captured by the POMDP
tuple of \cref{sec:prelim-pomdp} and not as an axiom. Probabilistic
internalisation and predictive sufficiency are our
Axioms~\ref{ax:probabilistic} and \ref{ax:sufficient}. Recursive
updatability is our Axiom~\ref{ax:recursive}. Evidence
monotonicity is derivable from the Bayes filter and is now a
consequence of \cref{thm:unique-update}. Ambiguity preservation
is \cref{thm:ambiguity} rather than an axiom. Policy invariance
is \cref{thm:policy-invariance}. Belief-action separation is our
Axiom~\ref{ax:policy}. The earlier set of eight postulates is
thus reduced to four independent axioms together with three
theorems.
\end{remark}

\section{Environment Specifications}
\label{app:environments}

This appendix gives the exact transition, observation, and reward
parameters implemented in the released environment code
(\texttt{tiger\_pomdp/environment.py} and
\texttt{red\_team\_graph/environment.py}), matching the summary in
\cref{sec:meth-envs}.

\subsection{Tiger POMDP}

$\Ssp = \{\text{left}, \text{right}\}$,
$\Asp = \{\text{listen}, \text{open-left}, \text{open-right}\}$,
$\Osp = \{\text{hear-left}, \text{hear-right}\}$, $\mu_0 =
(0.5, 0.5)$, $\gamma = 0.95$, $T_{\max} = 20$.

\begin{align*}
  T(\cdot \mid \cdot, \text{listen}) &= I_2, \\
  T(\cdot \mid \cdot, \text{open-left}) =
  T(\cdot \mid \cdot, \text{open-right}) &=
    \begin{bmatrix} 0.5 & 0.5 \\ 0.5 & 0.5 \end{bmatrix}, \\
  Z(\cdot \mid \cdot, \text{listen}) &=
    \begin{bmatrix} 0.85 & 0.15 \\ 0.15 & 0.85 \end{bmatrix}, \\
  Z(\cdot \mid \cdot, \text{open-left}) =
  Z(\cdot \mid \cdot, \text{open-right}) &=
    \begin{bmatrix} 0.5 & 0.5 \\ 0.5 & 0.5 \end{bmatrix}.
\end{align*}
Rows/columns are ordered (left, right). Reward
$R(\text{left}, \cdot) = (-1, -100, +10)$ and
$R(\text{right}, \cdot) = (-1, +10, -100)$ over
(listen, open-left, open-right).

\subsection{Red-Team Attack Graph ($K = 6$)}

$\Ssp = \{0,1\}^6$ ($|\Ssp| = 64$), one binary
(\emph{vulnerable}/\emph{hardened}) latent variable per node.
$\Asp = \{\text{scan}(i), \text{exploit}(i) : i \in
\{0,\dots,5\}\} \cup \{\text{wait}\}$; the \texttt{patch} action
family exists in the model definition but is excluded from the
action menu presented to every method, because its transition
dynamics are an unimplemented placeholder in the released
environment (\cref{sec:results-limitations}). $\Osp = \{0, 1\}$
(scan result), $\gamma = 0.95$, $T_{\max} = 30$, $\mu_0 =$ uniform
over $\Ssp$.

\emph{Observation model.} $\text{scan}(i)$ yields
$Z(1 \mid \text{state}, \text{scan}(i)) = 0.85$ if node $i$ is
vulnerable and $0.10$ if hardened (i.e.\ false-negative rate
$\beta = 0.15$, false-positive rate $\alpha = 0.10$); every other
action yields an uninformative observation,
$Z(o \mid \cdot, a) = 0.5$ for $a \neq \text{scan}(\cdot)$.

\emph{Transition model.} As implemented, $T(\cdot \mid s, a) =
\mathbf{1}[s]$ (the identity) for \emph{every} action $a$,
including \texttt{exploit}: node hardening by \texttt{patch} and
state change on \texttt{exploit} are unimplemented placeholders
(\texttt{environment.py} contains a \texttt{TODO} to this effect).
A node-dependency topology for lateral movement is declared but
not wired into the transition kernel:
\texttt{dependencies = \{0: [], 1: [], 2: [0], 3: [0,1], 4: [2,3],
5: [2,3]\}}, encoding an intended prerequisite structure (nodes 2
and 3 depend on node 0/1 being compromised first, nodes 4 and 5
depend on both 2 and 3) that the current release does not yet
enforce. Node vulnerability is instead drawn i.i.d.\ per episode.
This is the reason \textbf{AB1} (drop prediction step) shows only
a near-null effect on this environment (\cref{sec:results-attack}):
the operation it removes is already a no-op under the identity
kernel.

\emph{Reward.} $R(\text{scan}) = -0.5$,
$R(\text{exploit-success}) = +20$,
$R(\text{exploit-fail}) = -5$, $R(\text{wait}) = -0.1$; a repeat
exploit on an already-compromised node pays $0$ (no reward
farming). These match \cref{sec:meth-envs} except that
$R(\text{patch})$ is not applicable, since \texttt{patch} is
excluded from the action menu.

\section{Prompt Templates}
\label{app:prompts}

This appendix reproduces the exact system and user prompts issued
to the LLM by the evaluation scripts
(\texttt{tiger\_pomdp/run\_full\_eval.py} and
\texttt{red\_team\_graph/run\_full\_eval.py}), for the three
methods evaluated in \cref{sec:results-scope}. All three methods
within an environment share the same system prompt, temperature,
and parser (\cref{sec:meth-baselines}); only the state
representation in the user prompt differs. \texttt{\{\ldots\}}
denotes a value substituted at run time.

\subsection{Tiger POMDP}

\emph{System prompt.}
{\footnotesize
\begin{verbatim}
Select the action that maximizes the expected
immediate reward given the information provided
below. Break ties uniformly. Output your
reasoning, then a final line strictly formatted
as 'ACTION: <action>'.
\end{verbatim}
}

\emph{Reactive user prompt.}
{\footnotesize
\begin{verbatim}
DOMAIN: Tiger POMDP
GOAL: Choose the door without the tiger.
REWARDS: +10 for treasure, -100 for tiger, -1 for
listening.
AVAILABLE ACTIONS: listen, open-left, open-right

LATEST SENSOR OBSERVATION: {hear-left|hear-right}
\end{verbatim}
}

\emph{\BSE\ user prompt.}
{\footnotesize
\begin{verbatim}
DOMAIN: Tiger POMDP
GOAL: Choose the door with treasure while avoiding
the tiger.
REWARDS: +10 for treasure, -100 for tiger, -1 for
listening.
AVAILABLE ACTIONS: listen, open-left, open-right

CURRENT BELIEF POSTERIOR:
- tiger-left:  {b(left):.4f}
- tiger-right: {b(right):.4f}
\end{verbatim}
}

\emph{NL-Tracker user prompt.}
{\footnotesize
\begin{verbatim}
DOMAIN: Tiger POMDP
GOAL: Choose the door with treasure while avoiding
the tiger.
REWARDS: +10 for treasure, -100 for tiger, -1 for
listening.
AVAILABLE ACTIONS: listen, open-left, open-right

YOUR PREVIOUS BELIEF SUMMARY:
{belief text from previous step}

LATEST ACTION-OBSERVATION: action={action},
observation={observation}

INSTRUCTION: First write one line starting with
'BELIEF:' that states your updated belief in
natural language, including your best estimate of
the probability the tiger is on the left vs. the
right (e.g., 'BELIEF: ~80% left, ~20% right, based
on ...'). Then output a final line strictly
formatted as 'ACTION: <action>'.
\end{verbatim}
}

\subsection{Red-Team Attack Graph}

\emph{System prompt.}
{\footnotesize
\begin{verbatim}
Select the action that maximizes your expected
discounted return over the remaining steps of the
episode, not just the immediate reward -- scanning
now can be worth a short-term cost if it avoids a
larger expected loss later. Use the information
provided below. Break ties uniformly. Output your
reasoning, then a final line strictly formatted as
'ACTION: <action>'.
\end{verbatim}
}

\emph{Reactive user prompt.}
{\footnotesize
\begin{verbatim}
DOMAIN: K=6 Red-Team Attack Graph
GOAL: Discover vulnerable entry points and execute
successful exploits.
REWARDS: +20 for successful exploit, -5 for failed
exploit, -0.5 for scan, -0.1 for wait.
STEPS REMAINING: {steps_remaining}
AVAILABLE ACTIONS: scan_0, ..., scan_5, exploit_0,
..., exploit_5, wait

{Already compromised nodes, if any.}

LATEST SCAN OBSERVATION: Scanned node {i},
result={0|1} (1=flagged vulnerable, 0=clean).
\end{verbatim}
}

\emph{\BSE\ user prompt.} Identical header/reward/action lines to
the Reactive prompt above, replacing the observation line with:
{\footnotesize
\begin{verbatim}
CURRENT MARGINAL PROBABILITY OF EACH NODE BEING
VULNERABLE:
- Node 0: {p_0:.4f}
...
- Node 5: {p_5:.4f}
\end{verbatim}
}
where $p_i = \sum_{s : s_i = 1} b(s)$ is the per-node marginal of
the full 64-state posterior.

\emph{NL-Tracker user prompt.} Identical header/reward/action
lines, with the previous free-text belief and an instruction
analogous to the Tiger case, adapted to per-node probabilities:
{\footnotesize
\begin{verbatim}
YOUR PREVIOUS BELIEF SUMMARY:
{belief text from previous step}

LATEST SCAN OBSERVATION: {as above}

INSTRUCTION: First write one line starting with
'BELIEF:' that states, for each of the 6 nodes,
your estimated probability it is vulnerable (e.g.
'BELIEF: Node 0: 70%, Node 1: 20%, ...'). Then
output a final line strictly formatted as
'ACTION: <action>'.
\end{verbatim}
}

\section{Additional Experimental Details}
\label{app:experiments}

This appendix records the configuration actually used to produce
the numbers in \cref{sec:results}, as distinct from the full
designed protocol of \cref{sec:methodology}
(\cref{sec:results-scope} explains the gap between the two).

\emph{Model and decoding.} \texttt{gpt-4o} via the OpenAI API,
$\tau = 0.3$ for the main comparisons and ablations \textbf{AB1}/
\textbf{AB2}, $\tau = 1.0$ for \textbf{AB9}. On a malformed or
unparseable completion, the harness retries once at $\tau = 0.0$
before falling back to a fixed default action
(\texttt{listen} for Tiger, \texttt{wait} for the attack graph);
fallback calls are logged and counted in the reported abstention
rate but not excluded from the returns.

\emph{Sample sizes and concurrency.} Main comparisons use $N =
40$ paired episode seeds ($0,\dots,39$) per environment instance;
ablations use $N = 25$ seeds ($0,\dots,24$); a single
LLM-sampling seed per episode in both cases. Episodes are run
concurrently with a thread pool of 5 workers to keep live-API
wall-clock time tractable.

\emph{Bootstrap.} Percentile bootstrap on the mean, 2{,}000
resamples, fixed RNG seed 0. This is fewer than the 10{,}000
resamples specified for the full protocol in
\cref{sec:meth-metrics}, reduced for run time given the smaller
$N$; confidence intervals at $N = 40$ are wide regardless of
resample count, so this reduction has negligible effect on the
reported intervals.

\emph{Decision consistency.} Rather than the 200 pre-selected
belief-equivalent pairs with $K = 32$ resamples per pair specified
in \cref{sec:meth-metrics}, the reported numbers use an
opportunistic, reduced version: up to 10 (Tiger) or 8 (attack
graph) belief-collision groups mined from each method's own
main-run trajectories, resampled $K = 5$ times each. Tiger yielded
zero qualifying groups at $N = 40$ (reported as not measured);
the attack graph yielded 24 pairs.

\emph{Parser.} An exact-match parser scans the completion for a
line beginning \texttt{ACTION:} and accepts it only if the
remainder matches one of the environment's action names verbatim.

\emph{Per-episode compute.} Representative values from the logged
main-comparison runs (\texttt{logs/2026-09-02/}): on Tiger, mean
latency per episode is 4.1\,s (Reactive), 9.4\,s (\BSE), 0.9\,s
(NL-Tracker); on the attack graph, 111.6\,s (Reactive), 140.5\,s
(\BSE), 56.6\,s (NL-Tracker). The larger gap on the attack graph
reflects its longer horizon ($T=30$ vs.\ $T=20$) and larger
average token count per call.

\emph{Not implemented in this evaluation round.}
Chain-of-Thought, ReAct, QMDP, and POMCP baselines; the
open-weights replication; and ablations \textbf{AB3}--\textbf{AB8}
and \textbf{AB10}. No hyperparameters (POMCP simulation count,
UCB1 exploration constant, particle-filter size, etc.) were
therefore tuned or run for this round; the values quoted in
\cref{sec:meth-baselines} describe the intended configuration for
a future full run, not a completed one.

{\footnotesize
\emph{Raw logs.} Per-run JSON summaries and per-call logs are at
\texttt{logs/2026-09-02/\{tiger,attack\}\_\{full,ablations\}\_}
\texttt{eval\_results.json}, produced by the following scripts,
released with the code accompanying this preprint:
\begin{itemize}[leftmargin=*,noitemsep,topsep=2pt]
\item \texttt{tiger\_pomdp/run\_full\_eval.py}
\item \texttt{tiger\_pomdp/run\_ablations\_eval.py}
\item \texttt{red\_team\_graph/run\_full\_eval.py}
\item \texttt{red\_team\_graph/run\_ablations\_eval.py}
\end{itemize}
}

\end{document}